\documentclass{article}

\PassOptionsToPackage{numbers,sort,comma}{natbib}

\usepackage[main, final]{neurips_2026}

\usepackage[utf8]{inputenc}
\usepackage[T1]{fontenc}
\usepackage{url}
\usepackage{microtype}

\usepackage{amsmath}
\usepackage{amssymb}
\usepackage{amsfonts}
\usepackage{amsthm}      %
\usepackage{nicefrac}

\usepackage{graphicx}
\usepackage{xcolor}
\usepackage{colortbl}
\usepackage{float}
\usepackage{placeins}
\usepackage{subcaption}
\usepackage{wrapfig}
\usepackage{tikz}        %

\graphicspath{{./}}

\usepackage{booktabs}
\usepackage{threeparttable}
\usepackage{multirow}
\usepackage{multicol}
\usepackage{makecell}
\usepackage{array}
\usepackage{tabularx}
\usepackage{longtable}

\usepackage{enumitem}

\usepackage{algorithm}
\usepackage{algorithmic}

\usepackage{pifont}

\newcommand{\xmark}{\ding{55}}

\usepackage{xspace}

\usepackage{listings}
\lstdefinestyle{prompt}{
    breaklines=true,
    breakatwhitespace=false,
    basicstyle=\tiny\ttfamily,
    frame=none,
    columns=flexible,
    breakindent=0pt,
    literate={—}{{-}}1
}
\makeatletter
\newcommand*\myfontsize{%
  \@setfontsize\myfontsize{7}{8}%
}
\makeatother

\newtheorem{theorem}{Theorem}
\newtheorem{proposition}{Proposition}
\newtheorem{lemma}{Lemma}

\newtheorem{definition}{Definition}
\newtheorem*{problem}{Problem}
\newtheorem{assumption}{Assumption}

\usepackage{tcolorbox}
\usepackage{shortcuts}

\definecolor{PrimaryColor}{HTML}{51247A}
\definecolor{SecondaryColor}{HTML}{6F52E7}

\usepackage[colorlinks=true, linkcolor=SecondaryColor, citecolor=SecondaryColor, urlcolor=SecondaryColor]{hyperref}

\title{\textit{Catch-Only-One}: Non-Transferable Examples\\for Model-Specific Authorization}

\author{%
  Zihan Wang$^{1,2,*,\dagger}$ \quad
  Ethan (Zhiyong) Ma$^{1,\dagger}$ \quad
  Zhongkui Ma$^{1}$ \quad
  Shuofeng Liu$^{5}$ \\[1pt]
  \textbf{Akide Liu$^{6}$} \quad
  \textbf{Derui Wang$^{2}$} \quad
  \textbf{Minhui Xue$^{2,4}$} \quad
  \textbf{Guangdong Bai$^{3,*}$} \\[3pt]
  {\small\hspace*{-4pt}\makebox[281.5pt][s]{$^1$The University of Queensland\hfill
  $^2$CSIRO\hfill
  $^3$City University of Hong Kong}\hspace*{4pt}} \\
  {\small\hspace*{-4pt}\makebox[281.5pt][s]{$^4$RAIR, Adelaide University\hfill
  $^5$RMIT University\hfill
  $^6$Monash University}\hspace*{4pt}} \\
}

\begin{document}

\maketitle

\begingroup
\renewcommand{\thefootnote}{}
\footnotetext{$^{*}$Correspondence to: $\langle$\texttt{zihan.wang@uq.edu.au}$\rangle$, $\langle$\texttt{g.bai@cityu.edu.hk}$\rangle$. \quad $^{\dagger}$Core contributors.}
\endgroup

\vspace{-4pt}
\begin{abstract}
Recent AI regulations increasingly emphasize the need for mechanisms that preserve the utility of data for AI innovation while preventing misuse, particularly by enforcing \emph{purpose limitation} in downstream AI applications.
In practice, enforcing this principle remains challenging, as released data can be trivially fed into arbitrary models beyond its declared intent.
Existing approaches attempt to mitigate this risk by either perturbing data or retraining models to limit unintended use.
These strategies, however, offer no protection against inference by unknown or externally trained models, or fundamentally rely on control over the training or deployment.
In this work, we introduce \emph{\underline{n}on-\underline{t}ransferable \underline{e}xamples} (\codename{}s), recoded data that act as a task-level ``ciphertext'' decodable \emph{only} by a designated model.
Whereas adversarial examples exploit directions of high model sensitivity, \codename{}s leverage the complementary \emph{insensitive} subspace.
We propose a training-free, data-agnostic method that recodes data within a model-specific low-sensitivity subspace, preserving outputs for the authorized model while degrading unauthorized ones through subspace misalignment.
We establish formal bounds certifying authorized-model fidelity and showing that unauthorized degradation scales with measurable spectral misalignment between models.
Empirically, \codename{}s preserve performance across diverse vision backbones and state-of-the-art vision-language models under common preprocessing, while unauthorized models collapse even under adaptive reconstruction attacks.
These results establish \codename{}s as a practical means to preserve intended data utility while preventing unauthorized exploitation.

\par\smallskip
{\centering\small
\color{PrimaryColor}
\hypersetup{urlcolor=PrimaryColor}
\begin{tabular}{@{}r@{\enspace}l@{}}
\href{https://trusted-system-lab.github.io/model-specificity}{\faGlobe}\,: & \url{https://trusted-system-lab.github.io/model-specificity} \\[1pt]
\href{https://github.com/model-specific/non-transferable-examples}{\faGithub}\,: & \url{https://github.com/model-specific/non-transferable-examples}
\end{tabular}
\par}
\end{abstract}

\section{Introduction}
\label{sec:introduction}

\noindent\hspace*{0.75cm}%
\begin{minipage}[c]{1.1cm}
\includegraphics[width=\linewidth]{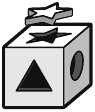}
\end{minipage}\hspace{0.4cm}%
\begin{minipage}[c]{\dimexpr\linewidth-3cm\relax}
\color{PrimaryColor}
\textit{The title alludes to Joseph Heller's {Catch-22}, which symbolizes an
unavoidable paradox. We adapt this notion in {Catch-Only-One} (or {Catch-11})
to capture the paradox that shared data may appear universally accessible yet
remains usable only by a single authorized model.}
\end{minipage}\par

Recent regulatory initiatives such as the EU AI Act~\cite{EU_AI_Act_2024} and the US AI Action Plan~\cite{USAIAction2025}, alongside industry frameworks such as DeepMind's Frontier Safety Framework~\cite{google25strength}, all enforce \emph{purpose limitation}: data must be used strictly in accordance with its declared purpose throughout the AI lifecycle, including as inputs to AI systems at inference or generation time.
Inputs such as prompts and uploaded documents must remain useful for authorized innovation while protected from misuse.

Despite this consensus, enforcement of purpose limitation is rarely satisfactory once data is released.
Existing data safeguards largely rely on \emph{administrative} controls (access restrictions, licenses, or terms of use) that cannot prevent shared data from being scraped, aggregated, and consumed by any AI application without consent.
A small number of paintings can suffice to replicate an artist's style~\cite{heikkila2022art,gal2022image}; medical scans shared for research can be exploited by membership inference attacks~\cite{shokri2017membership,hu2022m4i}.
The cost is tangible: a class-action lawsuit compelled Anthropic to pay over \$1.5B and erase pirated data~\cite{anthropic}, and European regulators have imposed roughly €100M in fines on Clearview AI for unlawful facial recognition matching uploaded face images against a huge unlicensed biometric database~\cite{noyb_clearview_criminal_complaint_2025}.

This situation calls for an enforceable \emph{post-release} guarantee: released data should remain fully usable \emph{only} by its intended recipient (the authorized model) while withholding utility from others.
The protection must not assume control over downstream training or deployment, since data owners cannot dictate who obtains the data or which models consume it, and it must be lightweight enough for real-world AI pipelines.
Three lines of existing research span a trade-off (lightweight protection that fails at inference, targeted suppression that requires retraining, or strong confidentiality at prohibitive cost), but none acts directly at inference on already-released data.
\emph{Data obfuscation}~\cite{ungeneralizable,wang2025provably} perturbs data prior to release so that standard learning pipelines degrade or fail to converge, but offers no guarantees once data is directly consumed at inference time.
\emph{Model-side restriction}~\cite{wang2022nontransferable,ijcai2025p1161} alters model objectives or parameters to suppress transfer in specific domains, requiring retraining with custom losses and control of the deployment pipeline.
\emph{Fully homomorphic encryption} (FHE)~\cite{gentry2009fully} enables inference on encrypted data, but its prohibitive computational and memory costs make it impractical for routine applications such as online media, healthcare workflows, or MLaaS services~\cite{lee2022privacy}.

\myparagraph{Our work}
Once data leaves its owner, it may be consumed by countless models. We recode released content into a task-level ``ciphertext''\footnote{We use ``ciphertext'' as an analogy: \codename{}s are obfuscated, model-specific recodings, not cryptographic encryption.} that only a designated authorized model $\fauth$ can decode, giving a \emph{post-release} technical control for purpose limitation.
We call the recoded inputs \emph{\underline{n}on-\underline{t}ransferable \underline{e}xamples} (\codename{}s), denoted $\tilde{x}$: the recoding constructs an authorization channel from insensitive directions of $\fauth$ and embeds model-specific signals that are benign to $\fauth$ but disruptive to others.

Our design exploits a structural property of neural networks: early layers are
sensitive along only a few directions, leaving a high-dimensional
\emph{insensitive subspace} specific to $\fauth$'s weights that misaligns
across models. This is the dual of the adversarial vulnerability behind adversarial examples~\cite{Su_2019,adv2015goodfellow}: where perturbations along \emph{sensitive} directions can drastically alter predictions, perturbations confined to the complementary insensitive set leave $\fauth$ undisturbed.
Using a small probe budget, we estimate a spectral basis for $\fauth$'s insensitive subspace and recode $x$ by adding a calibrated perturbation confined to this basis ($x\!\to\!\tilde{x}$).
By construction, \ntes preserve data utility for $\fauth$; the high dimensionality of the admissible set supports randomized per-sample recoding that disrupts transfer to unauthorized models and resists inversion empirically, including by adaptive adversaries trained on many $\tilde{x}$.
The procedure is training-free and data-agnostic, requiring no model modification and no knowledge of downstream models, and runs as a drop-in preprocessing step with negligible overhead, making it practical for MLaaS~\cite{ribeiro2015mlaas} deployments where user inputs are routinely replicated and repurposed~\cite{gal2022image}.

We prove two bounds for \codename{}s. For the authorized model, recoding perturbs
first-layer outputs by at most the spectral threshold $\tau$, negligible under
spectral flatness~\cite{vershynin2018high}.
For any unauthorized model, a cross-model bound shows degradation is
lower-bounded by the subspace misalignment $\alpha$, a quantity measurable
from model weights alone; the wider the architectural gap, the stronger the
separation guarantee.
Both bounds extend to the full network via Lipschitz composition, certifying
a prediction-level separation.
Empirically, we evaluate \codename{}s across five representative vision backbones
(\texttt{ResNet}, \texttt{ViT}, \texttt{SwinV2}, \texttt{DeiT},
\texttt{MambaVision}) and on \texttt{Qwen2.5-VL} and \texttt{InternVL3}
over MMBench~\cite{mmbench}, spanning perception and reasoning abilities
(coarse and fine-grained perception; attribute, relation, and logic reasoning).
The authorized model retains near-clean accuracy while unauthorized models
collapse to near-chance: at 0\,dB PSNR on ImageNet, authorized ResNet-50
top-1 stays at 80.2\% (from 80.3\% clean) while all others drop to
$\approx 0.1\%$.
\codename{}s also resist common preprocessing and reconstruction attacks.

\myparagraph{Contributions} We summarize our main contributions as follows:
\begin{itemize}[labelwidth=0.35cm, leftmargin=!]
\item \textbf{A new problem setting.} 
{We formulate \emph{model-specific data authorization}: user-provided data should remain fully usable for an authorized model while withholding utility from any other model. 
We introduce \emph{non-transferable examples (\codename{}s)}, a model-specific data representation that can be produced via a lightweight recoding process.}
\item \textbf{A formal framework and theoretical analysis.} We establish a strong formal foundation for \codename{}s with two bounds: authorized performance is preserved, while unauthorized output deviations are provably lower-bounded.
\item \textbf{An empirical evaluation.} Across diverse model architectures and data modalities, authorized models retain full utility while unauthorized models drop to near-chance despite common preprocessing, learned reconstruction, and supervised retraining attacks, with sub-millisecond encoding.
\end{itemize}

\section{Problem Formulation}
\label{sec:problem-formulation}

\subsection{Preliminaries}
\label{sec:neural-networks}

\myparagraph{Neural Networks} A feedforward network $f: \mathbb{R}^n \to \mathbb{R}^m$ alternates linear maps with pointwise nonlinearities~\cite{ma24popl}. Convolutional layers admit equivalent fully connected forms~\cite{wang2024corelocker} (Appendix~\ref{appendix:convolution}). Absorbing biases $b^{(i)}$ into $W^{(i)}$ by appending a constant $1$ to $x$ yields
\begin{gather*}
    y = f(x) = \phi(W^{(L)} \cdot \phi(W^{(L-1)} \cdot \cdots \phi(W^{(2)} \cdot \phi(W^{(1)} x)) \cdots)),
\end{gather*}
where $\phi$ is an activation function (\eg ReLU, sigmoid, tanh).

\begin{definition}[First-layer operator]
\label{def:first-layer}
For any feedforward network $f$ with first linear layer $W^{(1)}$ followed by nonlinearity $\phi$, the \emph{first-layer operator} is $W(f) := W^{(1)}$. We write $\Wauth := W(\fauth)$, with $\fauth$ formalized below. Convolutional first layers are instantiated by their unfolded matrix and vision-transformer first layers by their input projection (Appendix~\ref{appendix:convolution}).
\end{definition}

\myparagraph{Nullspace} The nullspace $\mathrm{Null}(W)=\{x\in\mathbb{R}^n \mid Wx=0\}$ of a linear map $W$ collects directions sent to zero, so any $\delta\in\mathrm{Null}(W)$ leaves the layer output invariant: $W(x+\delta)=Wx$.

\subsection{Model-Specific Data Authorization}
\label{sec:model_specific_authorization}

We consider a supervised task with ambient input space $\mathcal{X}$ and output space $\mathcal{Y}$.
A dataset $\mathcal{D}$ induces the task domains $\mathcal{X}_\mathcal{D}\subset\mathcal{X}$ and $\mathcal{Y}_\mathcal{D}\subset\mathcal{Y}$ observed in practice.
Training on $\mathcal{D}$ yields a family of models $\mathcal{F}_\mathcal{D}$, where each $f\in\mathcal{F}_\mathcal{D}$ implements a mapping $f:\mathcal{X}_\mathcal{D}\to\mathcal{Y}_\mathcal{D}$.
To quantify the \emph{usability} of $f$ on labeled inputs $(x,y)\sim P_\mathcal{D}$, we adopt a label-aware performance metric $\mu:\mathcal{F}_\mathcal{D}\times\mathcal{X}\times\mathcal{Y}\to[0,1]$ (\eg accuracy indicator for classification). The expected utility is $\expmet{f}{\mathcal{X}_\mathcal{D}} := \mathbb{E}_{(x,y)\sim P_\mathcal{D}}[\permet{f}{x}{y}]$.
Larger values of $\mu$ indicate better usability.

At inference on new data outside the training set, the data owner (defender) seeks to release content (\eg images, speech, documents) that is correctly processed by a designated authorized model $\fauth$ while remaining unusable to any unauthorized model $\funa\in\mathcal{F}_\mathcal{D}\setminus\{\fauth\}$.

We formulate \emph{model-specific data authorization}: preserve the usability of $\fauth$ while degrading that of any unauthorized model, without retraining and without assumptions on unauthorized models. Let $\tilde{x}=\mathcal{T}(x;\omega)$ be a randomized input recoding with $\mathcal{T}:\mathcal{X}\times\Omega\to\mathcal{X}$ and per-release seed $\omega\sim\mathbb{Q}$.
The recoded set $\Xrec := \{\mathcal{T}(x;\omega) : x \in \mathcal{X}_\mathcal{D},\, \omega\sim\mathbb{Q}\}$ defines \textit{non-transferable examples} (\codename{}s), a model-specific data representation tailored to $\fauth$; each $\tilde{x}$ inherits the label $y$ of its underlying clean $x$, and $\expmet{f}{\Xrec} := \mathbb{E}_{(x,y)\sim P_\mathcal{D},\,\omega\sim\mathbb{Q}}[\permet{f}{\mathcal{T}(x;\omega)}{y}]$. Formally:

\begin{problem}[$(\rho,\gamma)$-Model-Specific Data Authorization]
Given an authorized-utility tolerance $\rho\ge 0$ and an unauthorized separation margin $\gamma>0$, a randomized recoding $\mathcal{T}:\mathcal{X}\times\Omega\to\mathcal{X}$ is \emph{$(\rho,\gamma)$-model-specific} to $\fauth$ if, for every $\funa\in\mathcal{F}_\mathcal{D}\setminus\{\fauth\}$,
\begin{align}
\text{(authorized-utility retention)} \quad &
\expmet{\fauth}{\mathcal{X}_\mathcal{D}} - \expmet{\fauth}{\Xrec} \le \rho, \label{eq:retention}\\
\text{(unauthorized-utility degradation)} \quad &
\expmet{\fauth}{\Xrec} - \expmet{\funa}{\Xrec} \ge \gamma. \label{eq:degradation}
\end{align}
\end{problem}

\subsection{Threat Model}
\label{sec:problem-formulation:threat}

The data owner (defender), aligned with the data processor, has white-box access to $\fauth$ and uses task-domain data to instantiate the recoding $\mathcal{T}$. Model-specific parameters of $\mathcal{T}$ (\eg spectral basis, thresholds) remain private even when its algorithmic form is public. The recoding is lightweight and runs before release, producing $\tilde{x}=\mathcal{T}(x;\omega)$ that is shared or consumed downstream without disclosing $\fauth$. The defender has no access to unauthorized models.

The adversary operates \emph{unauthorized} models $\funa\in\mathcal{F}_\mathcal{D}\setminus\{\fauth\}$ and receives only $\tilde{x}$.
Before inference, they may apply an input-side preprocessing operator $\mathcal{A}\in\mathbb{A}$, where $\mathbb{A}$ is a class of standard preprocessing pipelines, yielding $\tilde{x}' := \mathcal{A}(\tilde{x})$.
The exact parameters of $\mathcal{T}$ are hidden.
The adversary's objective is to \emph{maximize} the task metric $\permet{\funa}{\tilde{x}'}{y}$ under fixed compute or query budgets. Specifically, we consider three adversary classes:
(\textit{i}) \emph{General Adversary \textbf{(GA)}:} any unauthorized $\funa$ with arbitrary architecture and parameters; (\textit{ii}) \emph{Transfer-match Adversary \textbf{(TA)}:} a GA sharing the architecture of $\fauth$ but with different weights; (\textit{iii}) \emph{Adaptive Adversary \textbf{(AA)}:} a GA/TA that optimizes the prediction pipeline, including learning $\mathcal{A}\in\mathbb{A}$ and updating $\funa$ (via \emph{fine-tuning} or \emph{retraining from scratch}), to maximize $\permet{\funa}{\tilde{x}'}{y}$.

\myparagraph{Scope}
This setting is practical as the enforcement process is executed locally by the data owner (or delegated to an authorized third party), requiring the model provider to disclose only minimal information. With white-box access to $\fauth$, the owner integrates $\mathcal{T}$ as an upstream service in the data pipeline, ensuring only $\tilde{x}$ is released.
In practice, this can be deployed as a middleware gateway in commercial pipelines (\eg an MLaaS front-end) that recodes inputs prior to storage or sharing, requiring no changes to downstream training procedures or model weights. From a governance perspective, this establishes an enforceable \emph{ex ante} technical control for purpose-limitation obligations in settings where \emph{ex post} monitoring of downstream reuse is intractable.
Neither the weights of $\fauth$ nor the model-specific parameters of $\mathcal{T}$ are disclosed to downstream recipients. 
This work focuses on an instantiation with an elegant closed-form algorithm and formal bounds; well-studied query-based sensitivity estimation~\cite{cai2024gradientlike,wycoff2021sequential} offers a route to black-box instantiations outside the present scope.

\myparagraph{Design Justifications}
Limited white-box access to $\fauth$ (\eg gradients or partial weights) is highly consistent with delegated auditing and transparency obligations reflected in EU AI Act Annex VII~\cite{EU_AI_Act_2024}, SR~26-2~\cite{sr26_2}, and GDPR Art.~15~\cite{gdpr_1} (Appendix~\ref{appendix:wb-justification}). \codename{}s are also deliberately not human-readable: $\tilde{x}$ acts as a task-level ``ciphertext'' that only $\fauth$ can decode (Appendix~\ref{appendix:visual-effect}).

\section{Our Method}
\label{sec:method}

Neural networks typically begin with a linear feature extractor (\eg a convolution with weight sharing or a token embedding), which reduces redundancy because input coordinates are often correlated (see Appendices~\ref{appendix:pca} and \ref{appendix:svd}).
This motivates \emph{input-side} perturbations that lie in an \emph{insensitive} subspace of the authorized model's first linear map: these directions are nearly inert for the authorized model but, due to subspace misalignment across models, can induce nontrivial changes for others.

\subsection{Insensitive Subspace Identification}
Let $W := \Wauth$ denote the first linear transformation of the authorized model (Definition~\ref{def:first-layer}, bias omitted). Since input dimensionality is typically high, $W$ has a nontrivial nullspace, and any $\delta\in\mathrm{Null}(W)$ satisfies $W(x+\delta)=Wx$, hence $f(x+\delta)=f(x)$.

We relax exact nulling to controlled feature deviation, requiring $W\delta\approx 0$ rather than $W\delta=0$; perturbations in such an insensitive subspace induce only small first-layer changes, especially under intervening nonlinearities (\eg ReLU truncation, sigmoid/tanh saturation). For $W\in\mathbb{R}^{m\times n}$ with singular value decomposition (SVD) $W=U S V^\top$, the nullspace is spanned by right singular vectors with zero singular values; we extend this to the \emph{$\tau$-insensitive subspace}, the span of right singular vectors whose singular values are at most a threshold $\tau>0$. We formalize it as follows.

\begin{definition}[$\tau$-insensitive Subspace]
\label{def:insensitive-subspace}
    Let $W \in \mathbb{R}^{m \times n}$ have SVD, $W = U S V^\top$, with singular values ordered $s_1 \le s_2 \le \cdots \le s_n$.
    Given a spectral threshold $\tau>0$, the $\tau$-insensitive subspace is
    \begin{gather}
        \Inst{W}
        = \mathrm{span}\{ v_1, v_2, \dots, v_k \},
    \end{gather}
    where $k=\max\{ i \mid s_i \le \tau \}$, and $v_i$ denotes the $i$-th column of \,$V$ corresponding to $s_i$ in $S$.
\end{definition}

\noindent This construction naturally {captures} the nullspace, and $\Inst{W} \supseteq \mathrm{Null}(W)$ for any $\tau \ge 0$.

\subsection{Non-Transferable Examples}
\label{sec:method:ne_construct}
Inheriting Definition~\ref{def:insensitive-subspace}, let $\Vsmall \in \mathbb{R}^{n \times k}$ collect the right singular vectors of $W$ corresponding to singular values $s_i \le \tau$, so $\mathrm{range}(\Vsmall) = \Inst{W}$.
We sample a vector $z \in \mathbb{R}^k$ (\eg \textit{i.i.d.} Gaussian, structured pattern, or content-dependent code) and lift it into $\Inst{W}$ via $\delta = \Vsmall z$.
The perturbation $\delta$ is added to the original input $x$ to obtain a recoded input $\tilde{x} = x + \delta$, which forms an \codename{}.
Since $\Vsmall$'s columns are orthonormal and their associated singular values are bounded by $\tau$, $\|W\delta\|_2 \le \tau \|z\|_2$.
By the distributive law of matrix multiplication, $W \tilde{x} = W (x + \delta) = W x + W \delta$, where $W \delta$ is small and thus $W \tilde{x}$ is close to $W x$. Figure~\ref{fig:misalignment-2d} visualizes this construction. For any unauthorized model $\funa$ with first-layer operator $W'$, the same $\delta$, inert to $\fauth$'s sensitive directions, generally lands on $W'$'s sensitive directions whenever the two right singular bases disagree, previewing the cross-model degradation formalized in Section~\ref{sec:cross-model-degradation}.

\begin{figure}[t]
\centering
\begin{subfigure}[c]{0.32\linewidth}
\centering
\begin{tikzpicture}[scale=1.45,>=stealth,font=\footnotesize]
  \useasboundingbox (-1.02,-0.30) rectangle (1.98,1.56);
  \def\rotang{35}
  \def\xang{22}
  \def\xlen{0.75}
  \node[black!60, font=\scriptsize] at (0.80,1.45) {Original input};
  \fill (0,0) circle (0.022);
  \draw[->, blue!75!black, opacity=0.25, line width=0.35pt] (0,0) -- (0,1.37);
  \node[blue!75!black, opacity=0.4, above] at (0,1.37) {$v_1$};
  \draw[->, blue!75!black, opacity=0.8, line width=0.6pt] (0,0) -- (1.37,0);
  \node[blue!75!black, right] at (1.37,0) {$v_2$};
  \draw[->, red!75!black, opacity=0.25, line width=0.35pt] (0,0) -- ({-1.37*sin(\rotang)},{1.37*cos(\rotang)});
  \node[red!75!black, opacity=0.4, above] at ({-1.37*sin(\rotang)},{1.37*cos(\rotang)}) {$v'_1$};
  \draw[->, red!75!black, opacity=0.8, line width=0.6pt] (0,0) -- ({1.37*cos(\rotang)},{1.37*sin(\rotang)});
  \node[red!75!black, right] at ({1.37*cos(\rotang)},{1.37*sin(\rotang)}) {$v'_2$};
  \coordinate (xtip) at ({\xlen*cos(\xang)},{\xlen*sin(\xang)});
  \draw[->, black, line width=1.5pt] (0,0) -- (xtip);
  \node[black, right, yshift=-2pt] at (xtip) {$x$};
  \coordinate (fv2) at ({\xlen*cos(\xang)},0);
  \draw[dashed, blue!75!black, opacity=0.45, line width=0.3pt] (xtip) -- (fv2);
  \fill[blue!75!black] (fv2) circle (0.019);
  \begin{scope}[shift={(fv2)}]
    \draw[blue!75!black, opacity=0.6, line width=0.3pt] (-0.045,0) -- (-0.045,0.045) -- (0,0.045);
  \end{scope}
  \node[blue!75!black, opacity=0.8, below, font=\scriptsize] at (fv2) {$\langle v_2, x\rangle$};
  \coordinate (fvp2) at ({\xlen*cos(\xang-\rotang)*cos(\rotang)},{\xlen*cos(\xang-\rotang)*sin(\rotang)});
  \draw[dashed, red!75!black, opacity=0.45, line width=0.3pt] (xtip) -- (fvp2);
  \fill[red!75!black] (fvp2) circle (0.019);
  \begin{scope}[shift={(fvp2)}, rotate=\rotang]
    \draw[red!75!black, opacity=0.6, line width=0.3pt] (-0.045,0) -- (-0.045,-0.045) -- (0,-0.045);
  \end{scope}
  \node[red!75!black, opacity=0.8, above left, font=\scriptsize, xshift=5pt, yshift=2pt] at (fvp2) {$\langle v'_2, x\rangle$};
\end{tikzpicture}
\end{subfigure}
\hfill
\begin{subfigure}[c]{0.32\linewidth}
\centering
\begin{tikzpicture}[scale=1.45,>=stealth,font=\footnotesize]
  \useasboundingbox (-1.02,-0.30) rectangle (1.98,1.56);
  \def\rotang{35}
  \def\xang{22}
  \def\xlen{0.75}
  \def\dlen{0.85}
  \node[black!60, font=\scriptsize] at (0.80,1.45) {Add $\delta$ along $v_1$};
  \shade[bottom color=forestgreen!2, top color=forestgreen!18]
    (0,0) --
    ({\xlen*cos(\xang)},{\xlen*sin(\xang)}) --
    ({\xlen*cos(\xang)},{\xlen*sin(\xang) + \dlen}) --
    (0,\dlen) --
    cycle;
  \fill (0,0) circle (0.022);
  \draw[->, blue!75!black, opacity=0.25, line width=0.35pt] (0,0) -- (0,1.37);
  \node[blue!75!black, opacity=0.4, above] at (0,1.37) {$v_1$};
  \draw[->, blue!75!black, opacity=0.8, line width=0.6pt] (0,0) -- (1.37,0);
  \node[blue!75!black, right] at (1.37,0) {$v_2$};
  \draw[->, red!75!black, opacity=0.25, line width=0.35pt] (0,0) -- ({-1.37*sin(\rotang)},{1.37*cos(\rotang)});
  \node[red!75!black, opacity=0.4, above] at ({-1.37*sin(\rotang)},{1.37*cos(\rotang)}) {$v'_1$};
  \draw[->, red!75!black, opacity=0.8, line width=0.6pt] (0,0) -- ({1.37*cos(\rotang)},{1.37*sin(\rotang)});
  \node[red!75!black, right] at ({1.37*cos(\rotang)},{1.37*sin(\rotang)}) {$v'_2$};
  \coordinate (xtip) at ({\xlen*cos(\xang)},{\xlen*sin(\xang)});
  \coordinate (xttip) at ({\xlen*cos(\xang)},{\xlen*sin(\xang) + \dlen});
  \draw[->, black!40, line width=0.8pt, dashed] (0,0) -- (xtip);
  \node[black!50, right, yshift=-2pt] at (xtip) {$x$};
  \fill[blue!75!black] ({\xlen*cos(\xang)},0) circle (0.019);
  \fill[red!75!black] ({\xlen*cos(\xang-\rotang)*cos(\rotang)},{\xlen*cos(\xang-\rotang)*sin(\rotang)}) circle (0.019);
  \draw[->, forestgreen, line width=1.3pt] (0,0) -- (0,\dlen);
  \node[forestgreen, left] at (-0.02, 0.6*\dlen) {$\delta$};
  \draw[->, black, line width=1.5pt] (0,0) -- (xttip);
  \node[black, right] at (xttip) {$\tilde x$};
\end{tikzpicture}
\end{subfigure}
\hfill
\begin{subfigure}[c]{0.32\linewidth}
\centering
\begin{tikzpicture}[scale=1.45,>=stealth,font=\footnotesize]
  \useasboundingbox (-1.02,-0.30) rectangle (1.98,1.56);
  \def\rotang{35}
  \def\xang{22}
  \def\xlen{0.75}
  \def\dlen{0.85}
  \node[black!60, font=\scriptsize] at (0.80,1.45) {Recoded input};
  \fill (0,0) circle (0.022);
  \draw[->, blue!75!black, opacity=0.25, line width=0.35pt] (0,0) -- (0,1.37);
  \node[blue!75!black, opacity=0.4, above] at (0,1.37) {$v_1$};
  \draw[->, blue!75!black, opacity=0.8, line width=0.6pt] (0,0) -- (1.37,0);
  \node[blue!75!black, right] at (1.37,0) {$v_2$};
  \draw[->, red!75!black, opacity=0.25, line width=0.35pt] (0,0) -- ({-1.37*sin(\rotang)},{1.37*cos(\rotang)});
  \node[red!75!black, opacity=0.4, above] at ({-1.37*sin(\rotang)},{1.37*cos(\rotang)}) {$v'_1$};
  \draw[->, red!75!black, opacity=0.8, line width=0.6pt] (0,0) -- ({1.37*cos(\rotang)},{1.37*sin(\rotang)});
  \node[red!75!black, right] at ({1.37*cos(\rotang)},{1.37*sin(\rotang)}) {$v'_2$};
  \coordinate (xtip) at ({\xlen*cos(\xang)},{\xlen*sin(\xang)});
  \coordinate (xttip) at ({\xlen*cos(\xang)},{\xlen*sin(\xang) + \dlen});
  \draw[->, black!40, line width=0.8pt, dashed] (0,0) -- (xtip);
  \node[black!50, right, yshift=-2pt] at (xtip) {$x$};
  \coordinate (fv2) at ({\xlen*cos(\xang)},0);
  \draw[dashed, blue!75!black, opacity=0.45, line width=0.3pt] (xttip) -- (fv2);
  \fill[blue!75!black] (fv2) circle (0.019);
  \begin{scope}[shift={(fv2)}]
    \draw[blue!75!black, opacity=0.6, line width=0.3pt] (-0.045,0) -- (-0.045,0.045) -- (0,0.045);
  \end{scope}
  \node[blue!75!black, opacity=0.8, below, font=\scriptsize] at (fv2) {$\langle v_2, \tilde x\rangle = \langle v_2, x\rangle$};
  \coordinate (fvp2old) at ({\xlen*cos(\xang-\rotang)*cos(\rotang)},{\xlen*cos(\xang-\rotang)*sin(\rotang)});
  \fill[red!75!black] (fvp2old) circle (0.019);
  \coordinate (fvp2new) at ({(\xlen*cos(\xang-\rotang)+\dlen*sin(\rotang))*cos(\rotang)},{(\xlen*cos(\xang-\rotang)+\dlen*sin(\rotang))*sin(\rotang)});
  \draw[dashed, red!75!black, opacity=0.45, line width=0.3pt] (xttip) -- (fvp2new);
  \fill[red!75!black] (fvp2new) circle (0.019);
  \begin{scope}[shift={(fvp2new)}, rotate=\rotang]
    \draw[red!75!black, opacity=0.6, line width=0.3pt] (-0.045,0) -- (-0.045,0.045) -- (0,0.045);
  \end{scope}
  \node[red!75!black, opacity=0.8, font=\scriptsize, anchor=north west, xshift=4pt, yshift=-3pt] at (fvp2new) {$\langle v'_2, \tilde x\rangle$};
  \pgfmathsetmacro{\gaplen}{\dlen*sin(\rotang)}
  \begin{scope}[shift={(fvp2old)}, rotate=\rotang]
    \draw[<->, >=stealth, orange!90!red, line width=0.7pt] (0,-0.085) -- (\gaplen,-0.085);
    \draw[orange!90!red, line width=0.7pt] (0,-0.04) -- (0,-0.13);
    \draw[orange!90!red, line width=0.7pt] (\gaplen,-0.04) -- (\gaplen,-0.13);
  \end{scope}
  \draw[->, black, line width=1.5pt] (0,0) -- (xttip);
  \node[black, right] at (xttip) {$\tilde x$};
\end{tikzpicture}
\end{subfigure}
\par\vspace{1pt}
\begin{tikzpicture}[>=stealth, font=\scriptsize, every node/.style={text=black!75}]
  \begin{scope}[shift={(0,0)}]
    \draw[->, blue!75!black, opacity=0.25, line width=0.35pt] (0,0) -- (0,0.24);
    \draw[->, blue!75!black, opacity=0.8, line width=0.6pt] (0,0) -- (0.24,0);
  \end{scope}
  \node[anchor=base west] at (0.30, 0.03) {$\fauth$ (authorized): $v_1, v_2$};
  \begin{scope}[shift={(3.55,0)}]
    \draw[->, red!75!black, opacity=0.25, line width=0.35pt] (0,0) -- (0,0.24);
    \draw[->, red!75!black, opacity=0.8, line width=0.6pt] (0,0) -- (0.24,0);
  \end{scope}
  \node[anchor=base west] at (3.85, 0.03) {$\funa$ (unauthorized): $v'_1, v'_2$};
  \draw[->, forestgreen, line width=1.3pt] (7.45, 0.10) -- (7.83, 0.10);
  \node[anchor=base west] at (7.87, 0.03) {Recoding $\delta$};
  \draw[<->, >=stealth, orange!90!red, line width=0.7pt] (9.85, 0.10) -- (10.37, 0.10);
  \draw[orange!90!red, line width=0.7pt] (9.85, 0.03) -- (9.85, 0.17);
  \draw[orange!90!red, line width=0.7pt] (10.37, 0.03) -- (10.37, 0.17);
  \node[anchor=base west] at (10.41, 0.03) {Bounded gap};
\end{tikzpicture}
\caption{2D intuition: recoding $\delta$ along $\fauth$'s insensitive direction $v_1$ preserves the $v_2$-component (Thm.~\ref{thm:target-retention}) while sliding the $v'_2$-component by $\sin\theta\,\|\delta\|$ with $\theta=\angle(v_2,v'_2)$, degrading $\funa$ (Thm.~\ref{thm:cross-model-degradation}).}
\vspace{-10pt}
\label{fig:misalignment-2d}
\end{figure}

\myparagraph{Practical Implementation}
The procedure is architecture-agnostic. Fully connected fronts can directly use their first weight matrix as $W$.
For convolutional fronts, we apply the construction to the linearized operator of the first convolution (\eg \texttt{nn.Unfold}) to obtain $W$; after synthesis, the perturbation is folded back to the native input layout.
For Transformer-style text models (\eg BERT~\citep{devlin2019bert}), we take $W$ to be the input projection of the first multi-head self-attention block \emph{after} the embedding layer (\eg the \texttt{QKV} input projection or its concatenated form) and apply the same spectral construction.
Appendices~\ref{appendix:convolution} and \ref{appendix:token-embedding} establish the underlying linear equivalences. The generation of the vector $z$ can be stochastic (\eg seeded per instance) or deterministic (\eg fixed codes per class or per client).
We rescale $\delta$ by a factor $\lambda > 0$ that sets its amplitude.
The parameters $\tau$ and $\lambda$ govern the trade-off between authorized-utility retention and unauthorized-utility degradation: larger $\tau$ tends to increase the number of perturbed input dimensions, while larger $\lambda$ increases the perturbation magnitude.
In practice, the permissible insensitive space is deliberately loose, offering substantial flexibility in synthesis, as empirically explored in Section~\ref{sec:experiments}. 

\section{Theoretical Properties of \codename{}s}
\label{sec:theory}

We prove two output-norm bounds at the first-layer operator $W$: retention for the authorized model (Section~\ref{sec:target-retention}) and degradation for any unauthorized model (Section~\ref{sec:cross-model-degradation}). A network-wide propagation result (Proposition~\ref{prop:network-wide}) then lifts both through the full network via layer-wise composition, instantiating Eq.~\eqref{eq:retention} and supplying the norm-level separation underlying Eq.~\eqref{eq:degradation}. We defer full proofs to Appendix~\ref{appendix:deferred-proofs}, where the cross-model separation is derived from a quantifiable misalignment condition with explicit, finite-sample constants.

\myparagraph{Standing Assumptions and Notation}
Let $W = \Wauth$ (Definition~\ref{def:first-layer}) have SVD $U S V^\top$ with $s_1 \le \cdots \le s_n$; partition $V = [\Vsmall, V^{(l)}]$ at $\tau$, so $\Vsmall \in \mathbb{R}^{n \times k}$ spans $\Inst{W}$. The recoding of Section~\ref{sec:method:ne_construct} sets $z \sim \mathcal{N}(0, \sigma^2 I_k)$, $\delta = \Vsmall z$, $\tilde{x} = x + \delta$. For $\fauth = g_L \circ \cdots \circ g_1$ and unauthorized $\funa = g'_L \circ \cdots \circ g'_1$, each post-first $g_i$ is $L_i^\star$-Lipschitz and $g'_i$ is $\nu_i^{\prime\star}$-co-Lipschitz on the trajectory pair $(x,\tilde{x})$; write $\Lipprod := \prod_{i=2}^{L} L_i^\star$ and $\nuprod := \prod_{i=2}^{L} \nu_i^{\prime\star}$. Assumption~\ref{assu:spectral-flatness} is invoked where stated.

\subsection{Authorized-Utility Retention}
\label{sec:target-retention}

We first bound the first-layer deviation induced by the recoding $\tilde{x} = x + \delta$.

\begin{theorem}[Bounding Authorized Utility]
\label{thm:target-retention}
Under the standing assumptions, for any $t > 0$, with probability at least $1 - 2 k \sigma^4 / t^2$,
\begin{align}
    \|W\tilde{x} - W x\|_2 < \tau \sqrt{k \sigma^2 + t}.
\label{eq:target-retention-bound}
\end{align}
\end{theorem}
\begin{proof}
Orthonormality of $V$ gives $(\Vsmall)^\top W^\top W \Vsmall = (S^{(s)})^2$, with $S^{(s)} = \mathrm{diag}(s_1, \ldots, s_k)$ (zeros included). Since $s_q \leq \tau$ for $q \leq k$, $\|W\delta\|_2 = \|S^{(s)} z\|_2 \leq \tau \|z\|_2$. A Chebyshev tail bound on $\|z\|_2^2$ (mean $k\sigma^2$, variance $2k\sigma^4$) yields $\|z\|_2 < \sqrt{k\sigma^2 + t}$ at the stated probability, giving Eq.~\eqref{eq:target-retention-bound}; the Chebyshev computation that supplies the probability constant is in Appendix~\ref{appendix:proof-target-retention}.
\end{proof}

The threshold $\tau$ and noise scale $\sigma$ act as direct knobs on the first-layer deviation. Spectral flatness (Assumption~\ref{assu:spectral-flatness}) makes $\Inst{W}$ high-dimensional, so the bound is non-vacuous in practice; Figure~\ref{fig:svd} and Appendix~\ref{appendix:spectral-flatness} verify this empirically across the backbones we examine. Proposition~\ref{prop:network-wide} below lifts the bound to the full network and to expected utility.

\begin{figure*}[t]
    \centering
    \includegraphics[width=\linewidth]{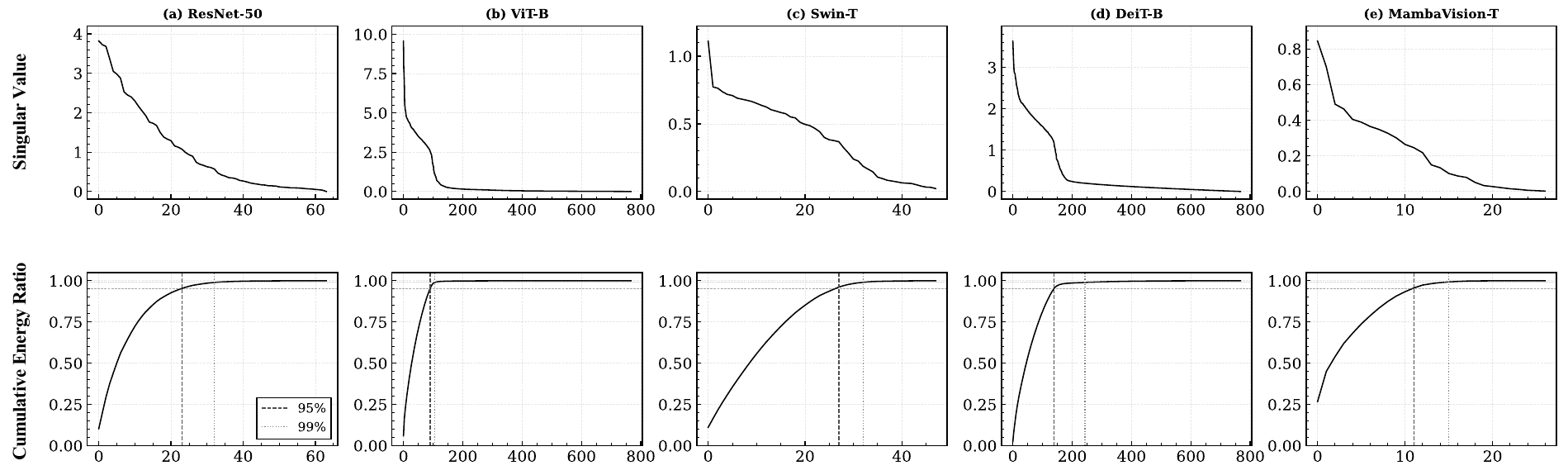}
    \caption{First-layer singular values (top) and cumulative energy share (bottom) across five models.}
    \label{fig:svd}
\end{figure*}

\subsection{Unauthorized-Utility Degradation}
\label{sec:cross-model-degradation}

The cross-model bound follows from a measurable misalignment between the authorized and unauthorized spectral bases. We assume the \emph{misalignment condition} $\|(V'^{(l)})^\top \Vsmall\|_2 \geq \alpha$ for some $\alpha > 0$, verified across all backbones. Figure~\ref{fig:misalignment-2d} (Section~\ref{sec:method}) gives the corresponding 2D picture, where $\alpha$ reduces to $\sin\theta$ for $\theta = \angle(v_2, v'_2)$ and generalizes it to higher dimensions.

\begin{theorem}[Bounding Unauthorized Utility]
\label{thm:cross-model-degradation}
Let $W' \in \mathbb{R}^{m \times n}$ be the first-layer operator of an unauthorized model with SVD $W' = U' S' V'^\top$ and singular values $s'_1 \le \cdots \le s'_n$. Partition $V' = [V'^{(s)}, V'^{(l)}]$ at index $k$, with $V'^{(l)} \in \mathbb{R}^{n \times (n-k)}$, and set $s'_{\min} := s'_{k+1}$. Under the standing assumptions and the misalignment condition, for any $t > 0$ and $c \in (0, \sqrt{\pi/2}]$, with probability at least $1 - 2k\sigma^4/t^2 - c\sqrt{2/\pi}$,
\begin{gather}
    \|W' \tilde{x} - W' x\|_2 - \|W \tilde{x} - W x\|_2 \;\geq\; c \alpha \, s'_{\min} \sigma - \tau \sqrt{k\sigma^2 + t}.
    \label{eq:cross-model-bound}
\end{gather}
\end{theorem}
\begin{proof}
Set $A := (V'^{(l)})^\top \Vsmall$. Orthogonality of $U'$ and the lower bound $s'_{k+i} \geq s'_{\min}$ give $\|W'\delta\|_2 \geq s'_{\min} \|A z\|_2$. Letting $p_1$ be the top right singular vector of $A$, $\|A z\|_2 \geq \|A\|_2\,|p_1^\top z|$ with $p_1^\top z \sim \mathcal{N}(0, \sigma^2)$; combining $\|A\|_2 \geq \alpha$ with Gaussian anti-concentration gives $\|A z\|_2 \geq c\alpha\sigma$ with probability at least $1 - c\sqrt{2/\pi}$. A union bound with Theorem~\ref{thm:target-retention} yields Eq.~\eqref{eq:cross-model-bound}; Gaussian anti-concentration on $p_1^\top z$ and joint-failure accounting against retention appear in Appendix~\ref{appendix:proof-cross-model-degradation}.
\end{proof}

The same recoding $\delta$ is suppressed by $W$'s small singular values and amplified by $W'$'s large ones whenever the two bases disagree on $\Inst{W}$. The size of the gap is governed by the misalignment $\alpha$ and the smallest sensitive gain $s'_{\min}$ of $W'$, both measured across architectures.

\myparagraph{Network-Wide Propagation}
We lift the first-layer bounds to the full network via layer-wise composition, then derive the expected-utility form of Eqs.~\eqref{eq:retention}, \eqref{eq:degradation}.

\begin{proposition}[Network-wide propagation]
\label{prop:network-wide}
Under the standing assumptions and the misalignment condition, with probability at least $1 - 2k\sigma^4/t^2 - c\sqrt{2/\pi}$ (for $t > 0$, $c \in (0, \sqrt{\pi/2}\,]$): (a) Lipschitz composition lifts Theorem~\ref{thm:target-retention} to a network-wide authorized-retention bound scaled by $\Lipprod$, and (b) co-Lipschitz composition lifts Theorem~\ref{thm:cross-model-degradation} to a network-wide cross-model separation scaled by $\nuprod$. Explicit forms (Eqs.~\eqref{eq:network-auth-bound} and~\eqref{eq:network-unauth-bound}) are deferred to Appendix~\ref{appendix:deferred-proofs}.
\end{proposition}
\begin{proof}
Part (a) follows by Lipschitz composition through $g_2, \ldots, g_L$ from Theorem~\ref{thm:target-retention}; part (b) by co-Lipschitz composition through $g'_2, \ldots, g'_L$ applied to Theorem~\ref{thm:cross-model-degradation}, joined by a union bound on the same draw $z$. Appendices~\ref{appendix:proof-network-auth},~\ref{appendix:proof-network-unauth} give the layer-wise induction and joint-failure union bound.
\end{proof}

The constants $L_i^\star$ and $\nu_i^{\prime\star}$ are \emph{local} to the trajectory pair $(x, \tilde{x})$. Appendix~\ref{appendix:layerwise} (Table~\ref{tab:layerwise_resnet}) confirms that the authorized deviations avoid the usual product explosion. Comparing the two parts of Proposition~\ref{prop:network-wide} yields the design rule $\nuprod\,c\alpha\,s'_{\min}\,\sigma > \Lipprod\,\tau\sqrt{k\sigma^2+t}$.
Connecting back to Eqs.~\eqref{eq:retention} and~\eqref{eq:degradation}, we specialize $\mu$ to top-1 accuracy: let $d(\fauth, x) := [\fauth(x)]_{\hat{y}} - \max_{y'\neq\hat{y}}[\fauth(x)]_{y'}$ be the logit margin ($\hat{y} = \arg\max_y [\fauth(x)]_y$). If the clean-data margin satisfies $\Pr_{x \sim P_\mathcal{D}}[d(\fauth, x) > \sqrt{2}\,\Lipprod\,\tau\sqrt{k\sigma^2+t}] \geq 1 - \eta$, then Proposition~\ref{prop:network-wide}(a) and a Cauchy--Schwarz step ($\|e_{\hat y} - e_{y'}\|_2 = \sqrt{2}$) together preserve the top-1 prediction under recoding; the failure-event mass then gives $\expmet{\fauth}{\mathcal{X}_\mathcal{D}} - \expmet{\fauth}{\Xrec} \leq \eta + 2k\sigma^4/t^2$, instantiating Eq.~\eqref{eq:retention} with $\rho = \eta + 2k\sigma^4/t^2$. Proposition~\ref{prop:network-wide}(b) yields the norm-level separation underlying Eq.~\eqref{eq:degradation}, realized as a prediction-level gap across architectures in Table~\ref{tab:fl_cross_matrix_one}.

\section{Experiments}
\label{sec:experiments}

\codename{}s preserve authorized utility while rendering unauthorized models unusable. Section~\ref{subsec:cross} establishes cross-model non-transferability against \textit{GA} and \textit{TA}; Section~\ref{subsec:comparison} positions \codename{}s against three authorization-related mechanisms; Section~\ref{subsec:practicality} stress-tests across modalities and against the adaptive adversary \textit{AA} via preprocessing, reconstruction, and learned adaptation.

\begin{wrapfigure}[21]{t}{0.55\textwidth}
    \centering
    \vspace{-\baselineskip}
    \includegraphics[width=\linewidth]{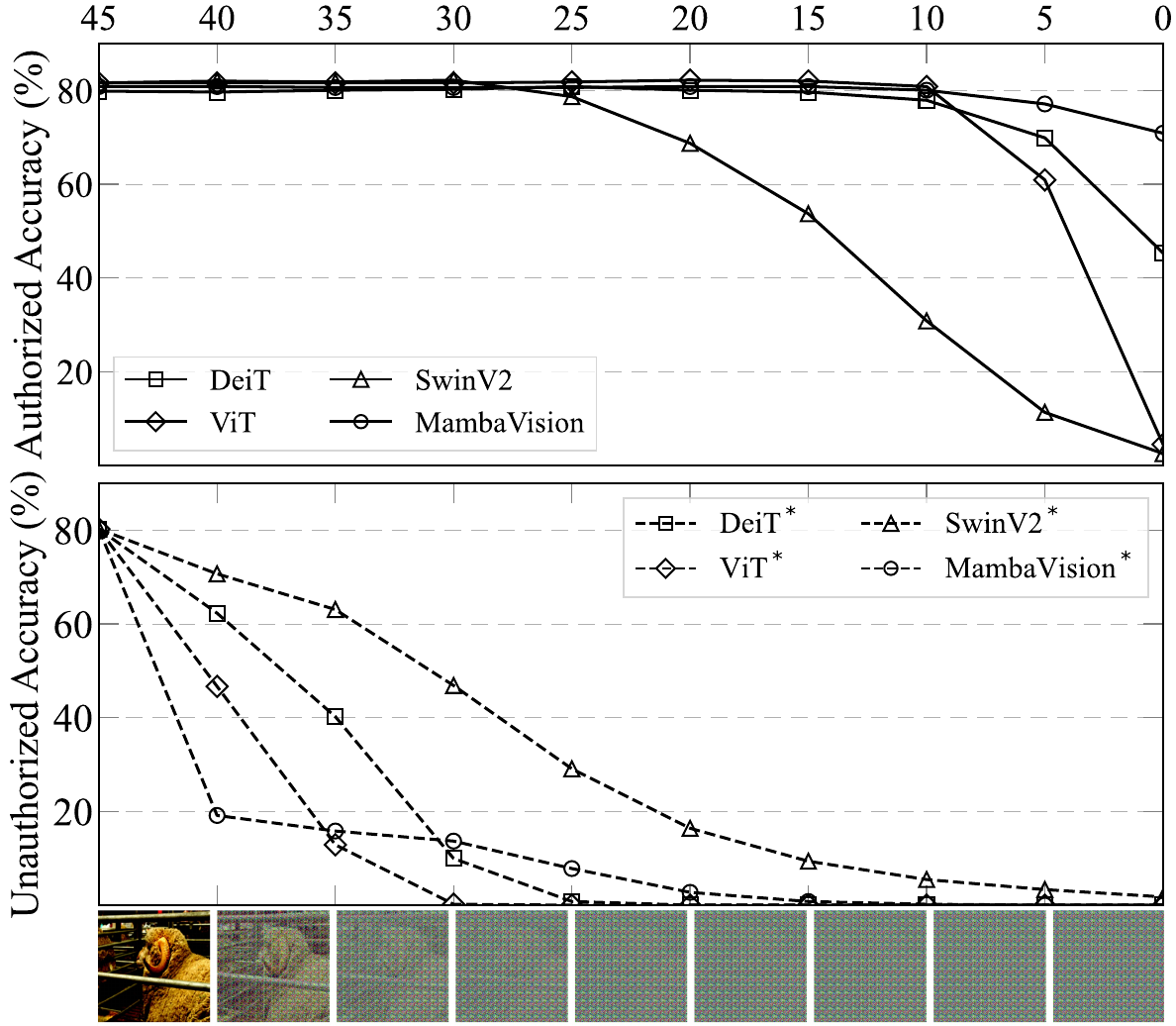}
    \caption{Authorized vs.\ unauthorized accuracy on target-recoded inputs across perturbation strength.}
    \label{fig:ne_curves}
\end{wrapfigure}

\myparagraph{Setup}
We evaluate five ImageNet pretrained backbones spanning classic and recent designs: \textit{ResNet-50}~\citep{he2016deep}, \textit{ViT-base-patch16-224}~\citep{dosovitskiy2020image}, \textit{SwinV2-tiny-patch4-window8-256}~\citep{liu2022swinv2}, \textit{DeiT-base-patch16-224}~\citep{deit}, and \textit{MambaVision-T-1K}~\citep{mambavision}. Experiments use CIFAR-10~\citep{krizhevsky2009learning} and ImageNet-1K~\citep{russakovsky2015imagenet}; CIFAR-10 results follow 10 epochs of fine-tuning from the ImageNet-1K checkpoints. ``Baseline'' rows in the tables denote clean accuracy under our evaluation pipeline. We also evaluate the generative vision-language models (autoregressive text decoders) \textit{Qwen2.5-VL-3B-Instruct}~\citep{qwenvl25} and \textit{InternVL3-1B}~\citep{InternVL3} on MMBench~\citep{mmbench}, which spans perception and reasoning abilities (coarse and fine-grained perception; attribute, relation, and logic reasoning).
We parameterize perturbation strength by a PSNR setting~\citep{hore2010image}\ (Appendix~\ref{appendix:visual-effect}); lower PSNR is stronger perturbation. For most backbones, unauthorized accuracy collapses around 20--25\,{dB} while authorized accuracy is largely preserved (Figure~\ref{fig:ne_curves}). We therefore use a conservative 20\,{dB} setting unless otherwise specified. Recoding samples $z$ with i.i.d.\ Gaussian entries and projects onto the $\tau$-insensitive directions; $\tau$ is set per backbone for Table~\ref{tab:fl_cross_matrix_one}, and Figure~\ref{fig:ne_curves} shows separation even with a shared $\tau$.

\subsection{Cross-Model Non-Transferability}\label{subsec:cross}

For each authorized target $\fauth$, we report authorized accuracy on $\tilde{x}=\mathcal{T}(x)$ and unauthorized accuracies on the same $\tilde{x}$. The sweep in Figure~\ref{fig:ne_curves} uses ImageNet, with ResNet-50 as the unauthorized comparator.
Table~\ref{tab:fl_cross_matrix_one} reports a $5\times5$ cross-architecture matrix on CIFAR-10 and ImageNet. At 20\,{dB} PSNR, \codename{}s keep authorized performance close to clean (\eg 98.8\% to 98.7\% on CIFAR-10; 80.3\% to 80.2\% on ImageNet), while unauthorized accuracy drops sharply, with off-diagonals at 5.5--20.6\% on CIFAR-10 and 0.0--12.9\% on ImageNet. SwinV2 shows the largest authorized drop on ImageNet (80.9\%$\to$71.7\%); Figure~\ref{fig:ne_curves} shows substantial separation between authorized and unauthorized accuracy, with clear room for further parameter optimization. Holding the architecture fixed but changing the weights gives accuracies of 8.2--21.0\% on CIFAR-10 and 0.0--1.2\% on ImageNet (Appendix~\ref{appendix:weight-specificity}, Table~\ref{tab:fl_cross_matrix_cross_only}).

\begin{table*}[!ht]
\centering
\caption{Cross-model accuracy at 20\,{dB} PSNR. Rows: authorized target $\fauth$; columns: evaluated model. \cgreen{Diagonals} are authorized; off-diagonals unauthorized.}
\label{tab:fl_cross_matrix_one}

\setlength{\tabcolsep}{2.31pt}
\renewcommand{\arraystretch}{0.9}

\tablescale{%
\begin{tabular}{lcccccccccc}
\toprule
&
\multicolumn{5}{c}{{CIFAR-10}} &
\multicolumn{5}{c}{{ImageNet}} \\
\cmidrule(lr){2-6}
\cmidrule(lr){7-11}

&
ResNet-50 & ViT-B & SwinV2-T & DeiT-B & MambaVision-T &
ResNet-50 & ViT-B & SwinV2-T & DeiT-B & MambaVision-T \\
\midrule

\rowcolor{gray!12}
Baseline &
98.2\% & 98.8\% & 96.1\% & 96.1\% & 96.9\% &
80.3\% & 81.6\% & 80.9\% & 79.9\% & 82.4\% \\

ResNet-50 &
\cgreen{\textbf{97.7}\%} &  12.5\% &  13.6\% & ~~9.4\% &  10.0\% &
\cgreen{\textbf{80.2}\%} & ~~0.0\% & ~~0.1\% & ~~0.1\% & ~~0.1\% \\
ViT-B    &
 10.5\% & \cgreen{\textbf{98.7}\%} & ~~9.3\% &  12.1\%  & ~~9.7\% &
~~0.0\% & \cgreen{\textbf{81.3}\%} & ~~0.0\% &  ~~0.0\% & ~~0.0\% \\
SwinV2-T   &
 15.6\% &  11.7\% & \cgreen{\textbf{88.4}\%} &  20.6\% & 18.8\% &
~~9.1\% & ~~4.3\% & \cgreen{\textbf{71.7}\%} & ~~7.4\% & 12.9\% \\
DeiT-B   &
~~9.4\% & ~~9.8\% & ~~9.8\% & \cgreen{\textbf{96.1}\%} & ~~5.5\% &
~~0.0\% & ~~0.0\% & ~~0.0\% & \cgreen{\textbf{79.3}\%} & ~~1.0\% \\
MambaVision-T  &
 17.6\% & ~~7.0\% & 13.7\% & 11.3\% & \cgreen{\textbf{94.5}\%} &
~~5.8\% & ~~0.0\% & ~~1.5\% & ~~0.7\% & \cgreen{\textbf{81.0}\%} \\
\bottomrule
\end{tabular}%
}
\end{table*}

\subsection{Positioning Among Authorization-Related Methods}\label{subsec:comparison}

\codename{}s provide inference-time model-specific authorization. We position them against three mechanisms with different threat models, not like-for-like baselines: Differential Privacy (DP)~\citep{dwork2006differential} bounds training-time leakage; Fully Homomorphic Encryption (FHE) provides confidentiality via encrypted inference; \textsc{AlgoSpec}~\citep{algospec}, the closest comparator, gates unauthorized inference via low-degree polynomial surrogates. Table~\ref{tab:accuracy_loss} reports ResNet-50 and ViT-B on CIFAR-10 and ImageNet.

\begin{table*}[!ht]
\centering
\caption{Positioning against authorization-related methods.}
\label{tab:accuracy_loss}

\setlength{\tabcolsep}{6.78pt}
\renewcommand{\arraystretch}{0.8}

\tablescale{%
\begin{threeparttable}
\begin{tabular}{l ccccc ccccc}
\toprule
&
\multicolumn{5}{c}{CIFAR-10} &
\multicolumn{5}{c}{ImageNet} \\
\cmidrule(lr){2-6}
\cmidrule(lr){7-11}

&
Plain & DP & FHE\tnote{3} & \textsc{AlgoSpec} & \codename{} (Ours) &
Plain & DP & FHE\tnote{3} & \textsc{AlgoSpec} & \codename{} (Ours) \\
\midrule
ResNet-50  &
98.2\% & 59.8\%\tnote{1} & 87.8\%\tnote{3} & ~~6.4\% & 97.7\% &
80.3\% & 63.1\%\tnote{1} & ~-- & ~~0.1\% & 80.2\% \\
ViT-B   &
98.8\% & ~--\tnote{1} & ~-- & 10.0\% & 98.7\% &
81.6\% & ~--\tnote{1} & ~-- & ~~0.0\% & 81.3\% \\
\midrule
Protection &
\xmark & \xmark\tnote{2} & ~~\checkmark & \xmark\tnote{2} & \checkmark &
\xmark & \xmark\tnote{2} & ~~\checkmark & \xmark\tnote{2} & \checkmark \\
\bottomrule
\end{tabular}
\begin{tablenotes}
\footnotesize
    \item[1] The evaluated DP pipeline struggles with batch norm and does not support multi-head attention in Transformers.
    \item[2] The authorized model performance is significantly impacted.
    \item[3] FHE inference at this scale is computationally prohibitive; reported figures are from the cited public ResNet-20 benchmark.
\end{tablenotes}
\end{threeparttable}%
}
\end{table*}

For DP training, we follow \citet{li2024you} using IBM \texttt{DiffPrivLib}~\citep{diffprivlib}. Authorized accuracy drops sharply, by more than 30\% on CIFAR-10 and around 20\% on ImageNet for ResNet-50, due to interactions with batch normalization. The pipeline also does not support transformer variants, limiting applicability. For FHE, we adopt a CKKS encrypted-inference configuration following TenSEAL~\citep{tenseal2021}; running it at our scale is computationally infeasible, with single-image latency exceeding 30 minutes. We therefore cite published ResNet-20 CIFAR-10 results~\citep{fheResnet}, which preserve clean authorized accuracy but incur prohibitive cost. For \textsc{AlgoSpec}, polynomial approximation of modern deep networks accumulates error with depth and width, driving authorized accuracy toward chance on both datasets. \codename{}s, by contrast, are a lightweight input-side recoding tied to the target model: under matched conditions they preserve authorized accuracy and drive mean unauthorized accuracy to chance on both datasets, with negligible overhead.

\subsection{Practicality of \codename{}s}\label{subsec:practicality}

\myparagraph{Generative Vision-Language Models} We extend the construction to generative vision-language models, with \texttt{InternVL3} as authorized and \texttt{Qwen2.5-VL} as unauthorized, evaluated on MMBench~\citep{mmbench} across capability dimensions AR, CP, FP-C, FP-S, LR, and RR. Table~\ref{tab:vlm} shows that authorized performance remains close to the clean baseline across all dimensions while unauthorized accuracy collapses (overall 78.8\% to 18.3\%). Figures~\ref{fig:vlm}, \ref{fig:vlm_safe}, and~\ref{fig:vlm_standard} show that the unauthorized model perceives \codename{}s as random noise.

\begin{table}[!ht]
\centering
\caption{Model-specific Authorization on MMBench for VLMs.}
\label{tab:vlm}

\setlength{\tabcolsep}{3.77pt}
\renewcommand{\arraystretch}{0.8}

\tablescale{%
\begin{tabular}{l ccccccc ccccccc}
\toprule
&
\multicolumn{7}{c}{InternVL3-1B \cgreen{(authorized)}} &
\multicolumn{7}{c}{Qwen2.5-VL-3B \cred{(unauthorized)}} \\
\cmidrule(lr){2-8}
\cmidrule(lr){9-15}

&
Overall & AR & CP & FP-C & FP-S & LR & RR &
Overall & AR & CP & FP-C & FP-S & LR & RR \\
\midrule

\rowcolor{gray!12}
Baseline &
72.7\% & 77.4\% & 82.4\% & 58.7\% & 79.4\% & 50.3\% & 66.8\% &
78.8\% & 81.3\% & 83.1\% & 69.2\% & 84.9\% & 66.5\% & 75.4\% \\
\codename{} (Ours) &
\cgreen{72.6\%} & \cgreen{77.8\%} & \cgreen{82.0\%} & \cgreen{58.3\%} & \cgreen{78.9\%} & \cgreen{50.9\%} & \cgreen{67.3\%} &
\cred{18.3\%} & \cred{29.2\%} & \cred{15.6\%} & \cred{17.0\%} & \cred{12.8\%} & \cred{17.3\%} & \cred{22.3\%} \\
\bottomrule

\end{tabular}%
}
\end{table}

\myparagraph{Robustness to Preprocessing} Common distortions leave the gap intact: on ViT-B/CIFAR-10, authorized \codename{}-accuracy stays within 1--3\,pp of clean under random 10\%--40\% crops and at 68.4\% under JPEG re-encoding to quality 75, with unauthorized models at random-guess in both settings; the same pattern holds across the VLM pipeline (Table~\ref{tab:vlm}, Appendix~\ref{subsec:vlm_preproc}).

\myparagraph{Adaptive Attacks} We evaluate two attackers against \codename{}s. An SR-ResNet~\citep{ledig2017photo} reconstruction attacker trained on recoded-to-recoded (Noise2Noise) or recoded-to-clean (Noise2Clean) pairs improves PSNR by at most $0.7\,\text{dB}$ (Appendix~\ref{subsec:reconstruction}). The stronger adversary \textit{AA} tries to \emph{learn around} \codename{}s by retraining an unauthorized model on encoded inputs. We grant \textit{AA} maximally favorable capabilities: clean labeled pairs $(x,y)$ from the same task distribution, and an arbitrarily large \emph{matched} encoded corpus produced by replicating the \codename{} pipeline and pairing each $\tilde{x}$ with the original label. This removes the usual obstacles to adaptation, namely distribution shift, label scarcity, and limited encoded data; the attacker trains directly on the inference-time format under full supervision.

\begin{wraptable}{r}{0.62\textwidth}
\centering
\caption{Learned adaptation attack against \codename{}s.}
\label{tab:nte_results}
\setlength{\tabcolsep}{4pt}
\renewcommand{\arraystretch}{0.95}
\tablescale{%
\begin{tabular}{lccccc}
\toprule
Attacker Model & ResNet-50 & ViT-B & SwinV2-T & DeiT-B & MambaVision-T \\
\midrule
Clean                & 98.2\% & 98.8\% & 96.1\% & 96.1\% & 96.9\% \\
Attacker             & 11.1\% & 10.9\% & 10.5\% & 13.8\% & 10.9\% \\
\codename{}-Attacker & 10.3\% & 10.7\% & 10.3\% & 10.9\% & 11.2\% \\
\bottomrule
\end{tabular}%
}
\end{wraptable}

Table~\ref{tab:nte_results} summarizes the results. All five attacker backbones evaluated on \codename{}s from the authorized ViT-B collapse to random-guess (row ``Attacker''), including a same-architecture ViT-B attacker: matching the architecture is not enough without the exact authorized weights. We then permit end-to-end adaptation by applying the same \codename{} procedure to every training example and fine-tuning each backbone with full-parameter updates for 3 epochs at batch size 64 and learning rate $5\times10^{-4}$. Even under this distribution-matched, fully supervised adaptation (row ``\codename{}-Attacker''), performance on \codename{}s remains at random-guess, showing that specificity holds against our strongest adversary.

\subsection{Runtime Analysis}
\label{sec:runtime}
\begin{wraptable}[11]{r}{0.34\textwidth}
\centering
\vspace{-\baselineskip}
\caption{\codename{} construction time.}
\label{tab:runtime}
\setlength{\tabcolsep}{4pt}
\renewcommand{\arraystretch}{0.95}
\tablescale{%
\begin{threeparttable}
\begin{tabular*}{\dimexpr5\linewidth/4\relax}{@{\extracolsep{\fill}}lS[table-format=3.2,table-number-alignment=center,mode=text]@{}}
\toprule
Model & {SVD (ms)$^1$} \\
\midrule
ResNet-50     & 1.8   \\
ViT-B         & 144.8 \\
SwinV2-T      & 0.7   \\
DeiT-B        & 148.3 \\
MambaVision-T & 0.2   \\
\midrule
Encoding / image$^2$ & \multicolumn{1}{c@{}}{\hspace*{-4pt}\setlength{\fboxsep}{1.5pt}\fbox{0.11\,ms}\hspace*{4pt}} \\
\bottomrule
\end{tabular*}
\begin{tablenotes}[flushleft]\small
\item[1] Compute \textbf{once} per model.
\item[2] For a \mbox{$224\times224\times3$} image.
\end{tablenotes}
\end{threeparttable}%
}
\end{wraptable}

\codename{} generation is lightweight enough for real-time inference and large-scale deployment. The only nontrivial step is a one-time SVD of the first-layer per-patch map $W \in \mathbb{R}^{m \times n}$ ($n{=}147$ for ResNet-50, $n{=}768$ for ViT-B/DeiT-B) at cost $O(\min\{m^2 n,\, m n^2\})$, amortized over all releases; recoding then samples $z \in \mathbb{R}^k$, computes $\delta = \Vsmall z$ at cost $O(nk)$, and sets $\tilde{x}=x+\delta$. On a single RTX A6000 (FP32, batch 1, host--device transfer included), encoding is sub-millisecond ($\approx 0.11$\,ms per $224\times224\times3$ image; Table~\ref{tab:runtime}), negligible against a forward pass; FLOPs and per-backbone breakdown in Appendix~\ref{appendix:runtime}.

\section{Related Work}
\label{sec:related}

\myparagraph{Watermarking and Content Provenance}
Watermarking embeds imperceptible signals to assert ownership or enable tracing~\cite{cox2002digital,hidden}, and provenance standards such as C2PA Content Credentials attach cryptographically signed, tamper-evident origin and edit metadata~\cite{c2pa2024techspec}. Yet watermarks can be removed or forged under realistic adversaries~\cite{bui2025trustmark} and, even when intact, only describe content rather than gate its use: a watermarked input remains consumable by any model.

\myparagraph{Learning-Time Defenses}
A complementary family intervenes during training: unlearnable~\cite{liu2024stable,wang2025provably} and ungeneralizable~\cite{ungeneralizable} examples add bilevel-crafted perturbations that block standard learners from generalizing, with the latter preserving learnability for a designated learner. Non-transferable training alters objectives or parameterizations to suppress cross-model or cross-domain transfer~\cite{wang2022nontransferable,ijcai2025p1161}. All presume training-pipeline influence and offer no leverage when models use released data at inference.

\myparagraph{Confidentiality Methods}
Differential privacy bounds any record's influence on the trained model~\cite{dwork2014algorithmic}, while fully homomorphic encryption keeps inputs encrypted under computation~\cite{gentry2009fully,folkerts2021redsec,lee2022privacy,kim2024privacy} at latency and memory costs precluding routine media or MLaaS workloads~\cite{lee2022privacy,fheResnet}. Neither directly provides model-specific authorization for released inputs.

\section{Discussion and Conclusion}
\label{sec:discussion}

We introduce \codename{}s, a training-free approach to model-specific data authorization that exploits the authorized model's insensitive directions. Our spectral analysis bounds output deviations and links cross-model separation to subspace misalignment. Experiments across five vision backbones and two vision-language models show that authorized utility is preserved while unauthorized models are disrupted, even under preprocessing, reconstruction, and adaptation attacks. After a one-time spectral computation, encoding takes less than a millisecond per image.

Future enforcement could estimate these directions through model queries alone, without direct access to model weights. Yet adversaries with similar query budgets might do the same. Our lightweight construction nevertheless allows data owners to vary the selected directions, perturbation strength, and random codes across releases. Beyond authorization, the same simple principle could find use in many other domains. One example is model-specific compression, where video or 3D data could be compressed by discarding information that the authorized model does not need.
\section{Ethical Considerations and Reproducibility}
\label{sec:ethics-reproducibility}

\subsection{Ethical Considerations}
\label{sec:ethics}
All experiments use only publicly available benchmark datasets (\eg CIFAR-10, ImageNet-1K, and MMBench) and standard pretrained checkpoints. We do not collect, infer, or expose personally identifiable information (PII), nor link data to real individuals. This work does not constitute human-subjects research and does not require IRB approval.

\myparagraph{Stakeholders}
The main stakeholders are\ \emph{data owners} seeking control over downstream reuse, \emph{data processors} (model owners or service providers) considering deployment, and the \emph{research and policy community} studying data governance.

\myparagraph{Intended Benefits and Potential Misuse}
\codename{}s are intended to restrict unauthorized secondary use of released inputs. They could also be used to impose opaque access restrictions or impede legitimate interoperability and auditing. Deployment should therefore specify who is authorized and provide clear disclosure and oversight.

\myparagraph{Harms, Fairness, and Limitations}
Our evaluation focuses on utility preservation and degradation across models and tasks, not demographic fairness. We introduce no new sensitive attributes and do not train on protected labels, but real-world deployment could interact with biases in the underlying model or dataset. Sending recoded inputs to the wrong model can cause silent failures. Deployments therefore need access control, model attestation, and clear interfaces.

\subsection{Reproducibility and Artifacts}
\label{sec:open-science}
\begin{enumerate}[leftmargin=*]
\item \textbf{Public datasets.}
We use only public datasets: CIFAR-10 and ImageNet-1K for image classification, MMBench for vision-language evaluation, and GLUE for language understanding. Each is obtained from its official source under its standard license.

\item \textbf{Source code.}
Code is publicly available at \url{https://github.com/model-specific/non-transferable-examples}.

\item \textbf{Models and checkpoints.}
We use publicly available pretrained checkpoints for vision backbones (ResNet-50, ViT-B, SwinV2-T, DeiT-B, MambaVision-T), vision-language models (InternVL3-1B and Qwen2.5-VL-3B), and language models (BERT and RoBERTa). We do not redistribute third-party datasets or checkpoints.
\end{enumerate}

\begin{ack}
We thank all anonymous reviewers for their constructive feedback.
We also thank Dr. Naipeng Dong and Dr. Dongge Liu for their insightful discussions.
Zihan Wang is in part supported by the Google PhD Fellowship.
\end{ack}

\bibliography{ref}

\appendix
\section{Supplementary Preliminaries}
\label{appendix:preliminaries}

\subsection{Eigendecomposition and PCA}
\label{appendix:pca}

Principal component analysis (PCA) projects data onto directions of maximal variance, obtained by eigendecomposing the data covariance.
We recall the eigendecomposition of a real symmetric matrix.
\begin{definition}[Eigendecomposition]
\label{def:eigendecomposition}
    The eigendecomposition of a real symmetric matrix $C \in \mathbb{R}^{n \times n}$ is a factorization of the form $C = V \Lambda V^\top$, where $V \in \mathbb{R}^{n \times n}$ is an orthogonal matrix whose columns are the eigenvectors of $C$, and $\Lambda \in \mathbb{R}^{n \times n}$ is a diagonal matrix whose diagonal entries are the eigenvalues of $C$.
    The eigenvalues are the scalars $\lambda_i$ such that $C v_i = \lambda_i v_i$, where $v_i$ is the $i$-th eigenvector of $C$.
\end{definition}
By default, eigenvalues are arranged in ascending order along the diagonal of $\Lambda$.
Geometrically, $V^\top$ rotates the coordinate system, $\Lambda$ scales the new axes by the eigenvalues, and $V$ rotates back.
For a centered random vector $x \in \mathbb{R}^n$ with $\mathbb{E}[x]=0$:
\begin{definition}[Principal Component Analysis (PCA)]
\label{def:pca}
    The covariance matrix is $C = \mathbb{E}[x x^\top] \in \mathbb{R}^{n \times n}$, where $\mathbb{E}$ denotes the expectation operator.
    The eigendecomposition of the covariance matrix is given by $C = V_{pca} \Lambda V_{pca}^\top$, where $V_{pca}$ is the matrix of eigenvectors (column vectors) and $\Lambda$ is the diagonal matrix of eigenvalues.
    The principal components are the columns of the matrix $V_{pca}$, and the projection of the data onto the principal components is given by $V_{pca}^\top x$.
\end{definition}
PCA commonly retains the top $k$ components for dimensionality reduction. The bottom components identify directions in which the input varies least.

\subsection{Singular Value Decomposition}
\label{appendix:svd}

We use the singular value decomposition (SVD) to analyze the spectral structure of weight matrices, following Section~\ref{sec:neural-networks} and omitting biases for simplicity.

\begin{lemma}[Singular Value Decomposition (SVD)]
\label{lemma:svd}
For any matrix $W \in \mathbb{R}^{m \times n}$, there exist orthogonal matrices
$U \in \mathbb{R}^{m \times m}$ and $V \in \mathbb{R}^{n \times n}$, and a diagonal matrix
$S \in \mathbb{R}^{m \times n}$ such that
\[
    W = U S V^\top.
\]
The diagonal entries of $S$ are the singular values of $W$, and the columns of $U$ and $V$ are the left and right singular vectors, respectively.
\end{lemma}

Singular values are non-negative, and we index them in ascending order along $S$ throughout this appendix.
Larger singular values correspond to higher-gain directions. Small singular values define the low-sensitivity directions used in our construction.

\subsection{Convolution}
\label{appendix:convolution}

Ignoring the bias, convolution is a linear map expressible as matrix multiplication.
We write the ungrouped convolution with unit stride and dilation, no padding, and no batch dimension.
\begin{definition}[Convolution]
\label{def:convolution}
    Let $X \in \mathbb{R}^{c_1 \times h \times w}$, $K \in \mathbb{R}^{c_1 \times c_2 \times k_h \times k_w}$, and $b \in \mathbb{R}^{c_2}$ denote the input, kernel, and bias, with $c_1$ input and $c_2$ output channels.
    The output $Y \in \mathbb{R}^{c_2 \times h' \times w'}$ has entries
    \begin{gather*}
        Y_{j, h'_i, w'_i} = \sum_{c=1}^{c_1} \sum_{u=1}^{k_h} \sum_{v=1}^{k_w} K_{c, j, u, v} \cdot X_{c, h'_i + u - 1, w'_i + v - 1} + b_{j},
    \end{gather*}
    where $j \in \{1, \ldots, c_2\}$ indexes the output channel and $(h'_i, w'_i)$ is the spatial position in the output.
\end{definition}

Unfolding the input into patches reduces convolution to matrix multiplication.
\begin{lemma}[Convolution as Matrix Multiplication]
\label{lem:convolution-matrix}
    Unfold $X$ into $X' \in \mathbb{R}^{c_1 k_h k_w \times h' w'}$ and reshape $K$ into $K' \in \mathbb{R}^{c_2 \times c_1 k_h k_w}$.
    Compute $Y' = K' X' + b$, broadcasting $b$ over spatial positions, and reshape $Y' \in \mathbb{R}^{c_2 \times h' w'}$ into $Y \in \mathbb{R}^{c_2 \times h' \times w'}$, which equals the convolution output in Definition~\ref{def:convolution}.
\end{lemma}
\begin{proof}
    Each local input patch becomes one column of $X'$, and each output-channel kernel becomes one row of $K'$. Multiplying $K'X'$ and adding the bias gives $Y'$, which is reshaped to $Y$.
\end{proof}

\myparagraph{Operator scope used in Section~\ref{sec:theory}}
The spectral construction uses a per-architecture local linear map $W_{\mathrm{loc}}$: the unfolded per-patch kernel $K' \in \mathbb{R}^{c_2 \times c_1 k_h k_w}$ of Lemma~\ref{lem:convolution-matrix} for convolutional fronts ($\mathbb{R}^{64\times 147}$ for ResNet-50's $7\!\times\!7$ stem), the patch-embedding projection for vision-transformer fronts, and the input weight matrix for fully connected fronts. Preprocessing (resize, clipping, channel normalization) is applied before recoding; $\delta$ is sampled in $W_{\mathrm{loc}}$'s standardized coordinates and lifted to the native input via the architecture-specific fold/tile/embed map. In the network-wide bounds, $W$ denotes the corresponding full input-to-feature operator, written $\bar W$ below.

For non-overlapping patch embeddings (ViT, DeiT, SwinV2; stride equals patch size), tiling is exact: a single $\delta_{\mathrm{patch}} \in \Inst{W_{\mathrm{loc}}}$ is replicated identically across all non-overlapping patches of the lifted $\delta$, so Theorem~\ref{thm:target-retention} applies per token. With $P=h'w'$ and the per-patch parameters $k,\sigma,t$, the tiled instantiation of Eq.~\eqref{eq:network-auth-bound} is
\[
\|\fauth(\tilde{x})-\fauth(x)\|_2 \leq \Lipprod\,\sqrt{P}\,\tau\sqrt{k\sigma^2+t},
\]
with the same probability as Theorem~\ref{thm:target-retention}. The logit-margin threshold in Section~\ref{sec:theory} is $\sqrt{2}$ times this full-tensor bound. When both models share the same aligned patch layout, both terms on the right-hand side of Eq.~\eqref{eq:network-unauth-bound} acquire the factor $\sqrt{P}$ under this local parameterization. For overlapping convolutional stems (\eg ResNet-50's $7\!\times\!7$ stride-$2$ conv and MambaVision's $3\!\times\!3$ stride-$2$ conv), the full image-to-feature map is the Toeplitz extension $\bar W = (I_{h'w'} \otimes K')\,U$ (with $U$ the unfold); local $\tau$-insensitivity bounds the per-patch deviation but does not automatically place every overlapping patch of the lifted $\delta$ in $\Inst{W_{\mathrm{loc}}}$. The retention bound thus lifts exactly for the three non-overlapping stems. For overlapping stems, the retention bound applies exactly to the local form used for authorized evaluation, and ResNet-50 layer-wise deviations stay at most 0.102 at 20\,dB (Appendix~\ref{appendix:layerwise}).

The cross-model $\alpha$ in Theorem~\ref{thm:cross-model-degradation} is measured \emph{pairwise} on a shared local domain (or, when first-layer domains differ across architectures, after a common pixel-space lift used uniformly in Section~\ref{sec:experiments}).

\subsection{Token Embedding}
\label{appendix:token-embedding}

Token embedding represents each vocabulary item by a learned vector~\citep{devlin2019bert}. With a one-hot input, it is a linear transformation.
\begin{definition}[Token Embedding]
\label{def:token-embedding}
    The token embedding is a linear transformation that maps a discrete token $t \in \{0,1\}^d$, encoded as a one-hot vector over a vocabulary of size $d$, to a continuous vector representation $e \in \mathbb{R}^m$ via a weight matrix $W \in \mathbb{R}^{m \times d}$, where $m$ is the embedding dimension.
    The token embedding is $e = W t$, so that $e$ selects the column of $W$ indexed by the active entry of $t$.
\end{definition}

\section{Supplementary Assumptions}
\label{appendix:assumptions}

We refer to the \textit{spectral distribution} of a matrix as its eigenvalues or singular values, as appropriate.

\paragraph{Standing constants.}
Throughout this appendix, we use two layer-wise composition constants of the authorized and unauthorized networks, evaluated on the trajectory of interest:
\begin{align*}
    \Lipprod := \prod_{i=2}^L L_i^\star, \qquad
    \nuprod := \prod_{i=2}^L \nu_i^{\prime\star},
\end{align*}
the authorized model's Lipschitz product and the unauthorized model's co-Lipschitz product, respectively.

\subsection{Spectral Flatness of the First-Layer Operator}
\label{appendix:spectral-flatness}

A spectrum is \textit{flat} when a nontrivial fraction of the singular values lie below a small threshold.

\begin{assumption}[Spectral Flatness of the First-Layer Operator]
\label{assu:spectral-flatness}
    Given a network trained on a specific task, there exist a threshold $\tau > 0$ and an index $k$ such that the bottom $k$ singular values of the first-layer weight matrix $W$
    \begin{gather*}
        s_1 \leq s_2 \leq \cdots \leq s_k \leq \tau
    \end{gather*}
    form a nontrivial fraction of the spectrum: the spectral mass below $\tau$ is bounded below by a constant $\beta > 0$ uniformly across architectures.
\end{assumption}
Figure~\ref{fig:svd} and Table~\ref{tab:app_svd_firstlayer} show low-energy tails in the first-layer spectra. The cumulative-energy curves are head-concentrated: ResNet-50 reaches 95.4\% of total spectral energy with 24 components and 99.0\% with 33, while ViT-B requires 91 and 107 components for the same levels despite a higher-dimensional patch projection. The corresponding spread is wide; minima are several orders of magnitude below maxima ($7.77\times 10^{-8}$ vs.\ $3.8203$ for ResNet-50; $3.68\times 10^{-5}$ vs.\ $9.5686$ for ViT-B), and medians sit well below means, consistent with Assumption~\ref{assu:spectral-flatness}.

\begin{table}[!ht]

\centering
\caption{Summary statistics of first-layer singular values.}
\label{tab:app_svd_firstlayer}

\renewcommand{\arraystretch}{1}

\tablescale{%
\begin{tabular}{lrrrr}
\toprule
& \makecell[c]{Max} & \makecell[c]{Min}~~~~ & \makecell[c]{Mean} & \makecell[c]{Median} \\
\midrule
ResNet-50 & 3.8203 & 7.77$\times$10\textsuperscript{$-$8} & 1.0242 & 0.5904 \\
ViT-B     & 9.5686 & 3.68$\times$10\textsuperscript{$-$5} & 0.5666 & 0.0482 \\
\bottomrule
\end{tabular}%
}
\end{table}

\subsection{Layer-wise Deviation under \codename{}s}
\label{appendix:layerwise}
Table~\ref{tab:layerwise_resnet} reports ResNet-50 layer-wise deviation norms at seven checkpoints (L0--L6, one every 40 forward operations counting normalizations and activations) for \codename{} strengths from 40 to 0\,dB PSNR. Deviations generally increase as PSNR decreases. At 0\,dB, the largest measured deviation is 1.5062 at L5, while the deviation at L6 is 0.0674.

\begin{table}[!ht]
    \centering
    \caption{Layer-wise deviation norms for ResNet-50 under different \codename{} strengths. L0--L6 index seven checkpoints, one every 40 forward operations (counting normalizations and activations).}
    \label{tab:layerwise_resnet}
    \tablescale{%
    \begin{tabular}{c ccccccc}
        \toprule
        PSNR (dB) & L0      & L1      & L2      & L3      & L4      & L5      & L6      \\
        \midrule
        40        & 0.0032  & 0.0097  & 0.0014  & 0.0062  & 0.0043  & 0.0210  & 0.0008  \\
        35        & 0.0035  & 0.0107  & 0.0017  & 0.0072  & 0.0050  & 0.0247  & 0.0010  \\
        30        & 0.0071  & 0.0200  & 0.0031  & 0.0146  & 0.0102  & 0.0498  & 0.0023  \\
        25        & 0.0134  & 0.0331  & 0.0052  & 0.0272  & 0.0184  & 0.0859  & 0.0040  \\
        20        & 0.0112  & 0.0367  & 0.0068  & 0.0282  & 0.0207  & 0.1016  & 0.0046  \\
        15        & 0.0240  & 0.0680  & 0.0104  & 0.0499  & 0.0349  & 0.1672  & 0.0079  \\
        10        & 0.0513  & 0.1731  & 0.0293  & 0.1108  & 0.0849  & 0.4344  & 0.0192  \\
        5         & 0.0754  & 0.3013  & 0.0501  & 0.1506  & 0.1239  & 0.6520  & 0.0264  \\
        0         & 0.1916  & 0.5643  & 0.0994  & 0.4031  & 0.2991  & 1.5062  & 0.0674  \\
        \bottomrule
    \end{tabular}%
    }
\end{table}

\section{Supplementary \texorpdfstring{Experiments and Discussion}{Experiments and Discussion}}
\label{appendix:experiments}

All experiments use Python 3.12.3 and PyTorch 2.3.0 (CUDA 12.3) on a workstation with an AMD Ryzen Threadripper PRO 5965WX (24 cores), 256\,GB RAM, and two NVIDIA RTX A6000 GPUs.

\subsection{Weight-level Specificity Matrix}
\label{appendix:weight-specificity}
Table~\ref{tab:fl_cross_matrix_cross_only} evaluates models with weights different from the authorized model: diagonal entries share its architecture, and shaded off-diagonals use different architectures. Diagonal accuracies are 8.2--21.0\% on CIFAR-10 and 0.0--1.2\% on ImageNet. Off-diagonal accuracies are also low.

\begin{table*}[!ht]
\centering
\caption{Model-specific non-transferability. All entries evaluate \codename{}s on weights different from the authorized model: diagonals share architecture, shaded off-diagonals are cross-architecture. Vertical labels give the training set of the authorized (row) model.}
\setlength{\tabcolsep}{1.62pt}
\renewcommand{\arraystretch}{0.9}

\tablescale{%
\begin{tabular}{@{}l c @{\hspace{1pt}}c@{} ccccc c @{\hspace{1pt}}c@{} ccccc@{}}
\toprule
&
\multicolumn{7}{c}{CIFAR-10} &&
\multicolumn{6}{c}{ImageNet} \\
\cmidrule(lr){3-8}
\cmidrule(lr){10-15}

&&
& ResNet-50 & ViT-B & SwinV2-T & DeiT-B & MambaVision-T &&
& ResNet-50 & ViT-B & SwinV2-T & DeiT-B & MambaVision-T \\
\midrule
ResNet-50 &
\multirow{5}{*}{\rotatebox[origin=c]{90}{ImageNet}} & \vrule &
 13.3\% & ~~\textcolor{gray!75}{9.4\%} & ~~\textcolor{gray!75}{8.2\%} & ~~\textcolor{gray!75}{9.4\%} & ~~\textcolor{gray!75}{7.8\%}   &
\multirow{5}{*}{\rotatebox[origin=c]{90}{CIFAR-10}} & \vrule &
 ~~1.2\% & ~~\textcolor{gray!75}{0.0\%} & ~~\textcolor{gray!75}{4.4\%} & ~~\textcolor{gray!75}{0.0\%} & ~~\textcolor{gray!75}{0.0\%} \\
ViT-B  && \vrule &
 \textcolor{gray!75}{10.1\%} & ~~9.6\% & \textcolor{gray!75}{10.1\%} & \textcolor{gray!75}{11.5\%} & ~~\textcolor{gray!75}{9.8\%} && \vrule &
 ~~\textcolor{gray!75}{0.0\%} & ~~0.0\% & ~~\textcolor{gray!75}{0.0\%} & ~~\textcolor{gray!75}{0.0\%} & ~~\textcolor{gray!75}{0.0\%} \\
SwinV2-T && \vrule &
 \textcolor{gray!75}{10.5\%} & \textcolor{gray!75}{10.2\%} & 21.0\% & \textcolor{gray!75}{11.3\%} & \textcolor{gray!75}{14.5\%} && \vrule &
 ~~\textcolor{gray!75}{0.0\%} & ~~\textcolor{gray!75}{0.0\%} & ~~0.0\% & ~~\textcolor{gray!75}{0.0\%} & ~~\textcolor{gray!75}{0.0\%} \\
DeiT-B && \vrule &
 \textcolor{gray!75}{12.5\%} & ~~\textcolor{gray!75}{9.3\%} & ~~\textcolor{gray!75}{7.0\%} & 14.8\% & \textcolor{gray!75}{10.9\%} && \vrule &
 ~~\textcolor{gray!75}{0.0\%} & ~~\textcolor{gray!75}{0.0\%} & ~~\textcolor{gray!75}{0.0\%} & ~~0.0\% & ~~\textcolor{gray!75}{0.0\%} \\
MambaVision-T&& \vrule &
 \textcolor{gray!75}{14.5\%} & ~~\textcolor{gray!75}{5.1\%} & \textcolor{gray!75}{12.5\%} & ~~\textcolor{gray!75}{9.8\%} & ~~8.2\% && \vrule &
 ~~\textcolor{gray!75}{0.0\%} & ~~\textcolor{gray!75}{0.0\%} & ~~\textcolor{gray!75}{1.9\%} & ~~\textcolor{gray!75}{0.0\%} & ~~0.0\% \\
\bottomrule

\end{tabular}%
}
\label{tab:fl_cross_matrix_cross_only}
\end{table*}

\myparagraph{Experiments on Language Models}
Table~\ref{table:glue_methods} reports GLUE~\citep{wang2019glue} results for BERT-base and RoBERTa-base~\citep{liu2019roberta}. Authorized scores remain within 2 points of their clean baselines. Unauthorized scores fall on every task, with RoBERTa retaining a score of 63.5 on STS-B.

\begin{table*}[!ht]

\centering
\caption{GLUE performance of authorized and unauthorized language models.}
\label{table:glue_methods}

\setlength{\tabcolsep}{4.877pt}
\renewcommand{\arraystretch}{0.95}

\tablescale{%
\begin{threeparttable}
\begin{tabular}{l *{14}{c}}
\toprule
&
\multicolumn{7}{c}{BERT-base} &
\multicolumn{7}{c}{RoBERTa-base} \\
\cmidrule(lr){2-8}
\cmidrule(lr){9-15}

&
CoLA & MNLI & QNLI & QQP & RTE & SST-2 & STS-B &
CoLA & MNLI & QNLI & QQP & RTE & SST-2 & STS-B \\
\midrule
\rowcolor{gray!12}
Baseline &
54.2 & 83.4 & 90.5 & 90.1 & 60.3 & 91.6 & 87.1 &
53.8 & 87.7 & 92.8 & 90.9 & 66.1 & 94.5 & 87.5 \\
\codename{} (Ours) &
54.5 & 82.9 & 89.5 & 89.6 & 60.3 & 89.6 & 86.9 &
55.5 & 87.5 & 92.4 & 90.8 & 65.3 & 94.6 & 87.2 \\
Unauthorized &
33.4 & 31.8 & 50.5 & 36.8 & 47.3 & 50.9 & 60.2 &
37.1 & 35.4 & 49.5 & 63.2 & 52.7 & 49.1 & 63.5 \\
\bottomrule
\end{tabular}
\end{threeparttable}%
}
\end{table*}

\subsection{Limitations}
\label{appendix:limitations}
Our setting grants the defender white-box access to $\fauth$ to estimate an insensitive subspace used for recoding. An attacker that estimates the recoding subspace can use an input-side projection to reduce the perturbation without direct access to $\fauth$. This may restore some unauthorized utility, but does not guarantee full recovery across models.

An attacker with training-time control may also recover utility by regularizing its model to be insensitive along the estimated subspace. Other practical limits include attenuation under aggressive acquisition pipelines or domain shift that changes early-layer geometry. Our analysis starts from the first-layer operator; extensions to deeper Jacobians and temporally coupled modalities remain open.

\subsection{Runtime Breakdown}
\label{appendix:runtime}
Table~\ref{tab:runtime} reports per-backbone wall-clock measurements (10 runs). The one-time SVD takes 0.2\,ms (MambaVision-T) to 148.3\,ms (DeiT-B) and is reused across releases. Per-image \codename{} encoding takes about $0.11$\,ms for a $224\times224\times3$ input across all backbones.

\myparagraph{Computational cost (ResNet-50)}
For the ResNet-50 stem, the local kernel matrix is $W_{\mathrm{loc}} \in \mathbb{R}^{64 \times 147}$. Computing a full right singular basis costs $\mathcal{O}(64\cdot147^2)$ operations. For a $224\times224$ RGB image, the projection and input addition cost $\mathcal{O}(nk+3\cdot224^2)$ operations, where $n=147$ and $k$ is the number of selected directions. For ViT-B and DeiT-B, $n=768$; measured encoding time remains below one millisecond per image (Table~\ref{tab:runtime}).

\subsection{Statistical Uncertainty}
\label{appendix:statistical-uncertainty}
For the fixed evaluation settings in Table~\ref{tab:fl_cross_matrix_one}, Hoeffding's inequality~\citep{vershynin2018high} and a union bound give simultaneous 95\% lower confidence bounds of 64.0 percentage points on CIFAR-10 ($n=10{,}000$) and 57.0 on ImageNet ($n=50{,}000$) for all 40 authorized--unauthorized accuracy gaps. These bounds account for independent per-image recoding and rounding of the reported accuracies.

\section{Deferred Proofs from Section~\ref{sec:theory}}
\label{appendix:deferred-proofs}
This section gives the full proofs of Theorems~\ref{thm:target-retention} and~\ref{thm:cross-model-degradation} and Proposition~\ref{prop:network-wide}, under the standing assumptions in Section~\ref{sec:theory}.

\subsection{\texorpdfstring{Full Proof of Theorem~\ref{thm:target-retention}: Chebyshev Tail on $\|z\|_2^2$}{Full Proof of Theorem 1: Chebyshev Tail on z-norm squared}}
\label{appendix:proof-target-retention}
The deterministic step in the main text gives $\|W\delta\|_2 \leq \tau\,\|z\|_2$. We complete the proof by bounding $\|z\|_2$ in probability.
For $z \sim \mathcal{N}(0, \sigma^2 I_k)$, $\mathbb{E}[\|z\|_2^2] = k\sigma^2$ and $\mathrm{Var}(\|z\|_2^2) = 2k\sigma^4$, so Chebyshev's inequality gives $\mathbb{P}[\|z\|_2^2 \geq k\sigma^2 + t] \leq 2k\sigma^4/t^2$. Therefore, with probability at least $1 - 2k\sigma^4/t^2$,
\begin{gather*}
    \|W\tilde{x} - Wx\|_2 = \|W\delta\|_2 \leq \tau\, \|z\|_2 < \tau\, \sqrt{k\sigma^2 + t},
\end{gather*}
which concludes the proof.

\subsection{\texorpdfstring{Full Proof of Theorem~\ref{thm:cross-model-degradation}: Gaussian Anti-Concentration on $p_1^\top z$}{Full Proof of Theorem 2: Gaussian Anti-Concentration}}
\label{appendix:proof-cross-model-degradation}
The main-text proof establishes
\begin{align*}
    \|W' \delta\|_2 \geq s'_{\min} \, \|(V'^{(l)})^\top \Vsmall z\|_2.
\end{align*}
We complete the proof by lower-bounding $\|(V'^{(l)})^\top \Vsmall z\|_2$ in probability.

Let $A := (V'^{(l)})^\top \Vsmall \in \mathbb{R}^{(n-k) \times k}$, and let $u_1, p_1$ be the top left/right singular vectors of $A$, so that $u_1^\top A = \|A\|_2 \, p_1^\top$. Then
\begin{align*}
    u_1^\top A z = \|A\|_2 \, p_1^\top z, \qquad p_1^\top z \sim \mathcal{N}(0, \sigma^2),
\end{align*}
since $\|p_1\|_2 = 1$ and $z \sim \mathcal{N}(0, \sigma^2 I_k)$. The density of $\mathcal{N}(0,\sigma^2)$ is bounded above by $1/(\sigma\sqrt{2\pi})$, so for any $c \in (0, \sqrt{\pi/2}]$,
\begin{align*}
    \Pr\bigl[|p_1^\top z| < c\sigma\bigr] \;\leq\; 2c\sigma \cdot \frac{1}{\sigma\sqrt{2\pi}} \;=\; c\sqrt{\tfrac{2}{\pi}},
\end{align*}
the standard Gaussian anti-concentration bound. Hence with probability at least $1 - c\sqrt{2/\pi}$,
\begin{align*}
    \|A z\|_2 \;\geq\; |u_1^\top A z| \;=\; \|A\|_2\, |p_1^\top z| \;\geq\; c\sigma\, \|A\|_2,
\end{align*}
where the first inequality is obtained by projecting onto the unit vector $u_1$. By the misalignment condition (standing assumption), $\|A\|_2 \geq \alpha$, so $\|Az\|_2 \geq c\alpha\sigma$, and therefore
\begin{align*}
    \|W' \delta\|_2 \geq s'_{\min}\, c\alpha\sigma.
\end{align*}
Taking the difference with the $W$ bound from Theorem~\ref{thm:target-retention} (which holds with probability at least $1 - 2k\sigma^4/t^2$) and union-bounding the two failure events yields Eq.~\eqref{eq:cross-model-bound}.

\subsection{Full Proof of Proposition~\ref{prop:network-wide}(a): Layer-Wise Induction}
\label{appendix:proof-network-auth}
Following the standing decomposition $\fauth = g_L \circ \cdots \circ g_2 \circ g_1$ with $g_1 = W$ the first linear operator, let $h_i = g_i \circ \cdots \circ g_1(x)$ and $\tilde{h}_i = g_i \circ \cdots \circ g_1(\tilde{x})$. By induction on $i$, we show $\|\tilde{h}_i - h_i\|_2 \leq (\prod_{j=2}^i L_j^\star) \cdot \tau \sqrt{k\sigma^2 + t}$.

For the base case $i=1$, Theorem~\ref{thm:target-retention} gives
\begin{align*}
\|\tilde{h}_1 - h_1\|_2 = \|\Wauth\delta\|_2 < \tau \sqrt{k\sigma^2 + t}.
\end{align*}
Now assume the claim holds for $i-1$. Since $g_i$ is $L_i^\star$-Lipschitz on the trajectory pair,
\begin{align*}
    \|\tilde{h}_i - h_i\|_2
    \leq L_i^\star \|\tilde{h}_{i-1} - h_{i-1}\|_2
    \leq \left(\prod_{j=2}^i L_j^\star\right) \tau \sqrt{k\sigma^2 + t}.
\end{align*}
At layer $L$, $\|\tilde{h}_L - h_L\|_2 = \|\fauth(\tilde{x}) - \fauth(x)\|_2$, so by induction
\begin{equation}
\|\fauth(\tilde{x}) - \fauth(x)\|_2 \;\leq\; \Lipprod\,\tau\sqrt{k\sigma^2 + t},
\label{eq:network-auth-bound}
\end{equation}
which is Proposition~\ref{prop:network-wide}(a).

\subsection{Full Proof of Proposition~\ref{prop:network-wide}(b): Layer-Wise Induction and Joint-Failure Union Bound}
\label{appendix:proof-network-unauth}
By Theorem~\ref{thm:cross-model-degradation}, $\|W'\delta\|_2 \geq c\alpha\, s'_{\min} \sigma$ with probability at least $1 - c\sqrt{2/\pi}$. Under the standing co-Lipschitz assumption (each $g'_i$ is $\nu_i^{\prime\star}$-co-Lipschitz on the trajectory pair), composition through $g'_2,\ldots,g'_L$ with constant $\nuprod$ propagates the lower bound, yielding
\begin{align*}
\|\funa(\tilde{x}) - \funa(x)\|_2 \;\geq\; \nuprod \, \|W'\delta\|_2 \;\geq\; \nuprod\, c\alpha\, s'_{\min} \sigma.
\end{align*}
Proposition~\ref{prop:network-wide}(a) gives $\|\fauth(\tilde{x}) - \fauth(x)\|_2 \leq \Lipprod \tau \sqrt{k\sigma^2 + t}$ with probability at least $1 - 2k\sigma^4/t^2$. Both events depend on the same draw $z$. A union bound gives a probability of at most $2k\sigma^4/t^2 + c\sqrt{2/\pi}$ that either event fails. Taking the difference on the joint success event yields
\begin{equation}
\|\funa(\tilde{x}) - \funa(x)\|_2 - \|\fauth(\tilde{x}) - \fauth(x)\|_2 \;\geq\; \nuprod\,c\alpha\,s'_{\min}\,\sigma - \Lipprod\,\tau\sqrt{k\sigma^2+t}.
\label{eq:network-unauth-bound}
\end{equation}

\section{Threat Model and Deployment\texorpdfstring{}{}}
\label{appendix:threat-model-justifications}

\subsection{Multi-Client Deployment}
\label{appendix:multi-client}
A data owner serving multiple authorized client models can construct a separate transformation $\mathcal{T}_{f_i}$ for each model $f_i$ and release $\tilde{x}_i=\mathcal{T}_{f_i}(x)$ to that client. The owner controls future issuance separately for each client. A shared transformation would require directions that remain insensitive for every authorized model. As the group grows, this requirement can shrink the available subspace and change the choice of $\tau$ and $\sigma$. Characterizing the resulting retention/degradation trade-off and constructing shared transformations are left to future work.

\subsection{Justification of White-Box Access}
\label{appendix:wb-justification}
Three regulatory frameworks motivate the ``delegated auditing'' paradigm for white-box access. The EU AI Act (Annex VII, \S4.5) grants notified bodies access to models and relevant parameters where necessary for high-risk conformity assessment, upon a reasoned request once all other reasonable means prove insufficient~\cite{EU_AI_Act_2024}. Under supervisory guidance (SR 26-2), banks were expected to validate their own use of vendor products~\cite{sr26_2}. GDPR Art.~15(1)(h) grants data subjects access to ``meaningful information about the logic'' of automated decision-making~\cite{gdpr_1}.

\subsection{On Visual Interpretability of \texorpdfstring{$\tilde{x}$}{x-tilde}}
\label{appendix:visual-effect}
Our design deliberately makes $\tilde{x}$ uninformative to humans. A human-readable $\tilde{x}$ would carry task-relevant information in a generic, model-agnostic form usable by arbitrary downstream models; \codename{}s instead retain utility through a channel tied to $\fauth$. In this sense, $\tilde{x}$ is a task-level ``ciphertext'' that only the authorized model can reliably ``decode''. The construction targets automated model inference. We define PSNR settings in the model's input space (the normalized inputs to the first-layer operator), using a peak value of 255 for all backbones. Because they differ in subspace dimension and input scale, we calibrate each sample's amplitude against a reference energy. The setting controls perturbation strength, and Figure~\ref{fig:sheep} shows how stronger recoding obscures a sample image. At 20\,dB, authorized accuracy is largely preserved by confining the perturbation to the insensitive subspace. When human review is required, the data owner can retain the clean input $x$ within the trusted boundary and release $\tilde{x}$ for downstream model inference.

\begin{figure}[!ht]
    \centering
    \includegraphics[width=\linewidth]{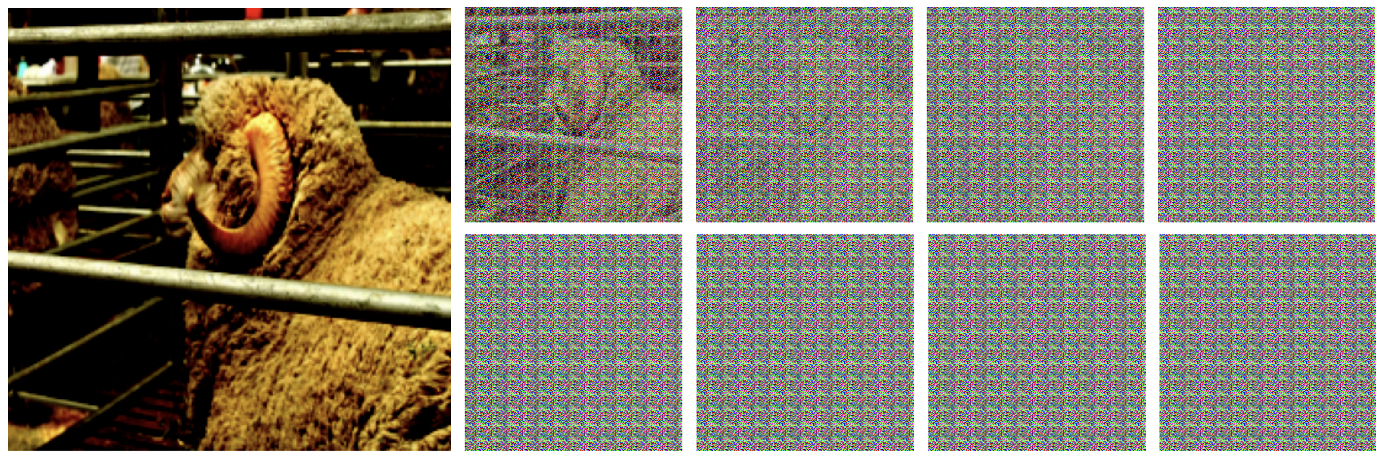}
    \caption{Visual examples under increasing perturbation strength. At 0\,{dB}, authorized ResNet-50 top-1 accuracy on ImageNet decreases by 0.1 percentage points.}
    \label{fig:sheep}
\end{figure}

\subsection{\texorpdfstring{Vision-Language Evaluation}{Vision-Language Evaluation}}\label{subsec:vlm_preproc}
InternVL3 and Qwen2.5-VL use model-specific image preprocessing. On MMBench, authorized InternVL3-1B accuracy changes from 72.7\% on clean inputs to 72.6\% on \codename{}s, while unauthorized Qwen2.5-VL-3B accuracy falls from 78.8\% to 18.3\%. The same separation appears across the six capability dimensions (AR, CP, FP-C, FP-S, LR, RR) in Table~\ref{tab:vlm}; Figure~\ref{fig:vlm} shows an example.

\newcommand{\thumb}[1]{\raisebox{-0.42\height}{\includegraphics[height=0.75cm]{#1}}}
\newcommand{\hair}{\arrayrulecolor{black!15}\midrule[0.3pt]\arrayrulecolor{black}}
\newlength{\chatlabw}
\newcommand{\chatsetup}{%
  \centering
  \scriptsize
  \setlength{\tabcolsep}{5pt}%
  \renewcommand{\arraystretch}{1.12}%
  \settowidth{\chatlabw}{\textsc{unauthorized}}%
}

\begin{figure*}[!ht]
\chatsetup
\begin{tabularx}{\linewidth}{@{}>{\raggedleft\arraybackslash}p{\chatlabw}XX@{}}
\toprule
\multicolumn{3}{@{}l@{}}{\textsc{Prompt}\enspace\textit{Please describe the image in detail.}}\\[2pt]
& \thumb{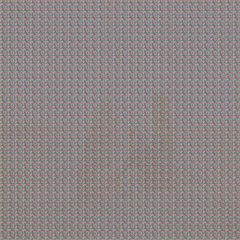}\enspace{\footnotesize(a)~\textsc{\codename{} input}}
& \thumb{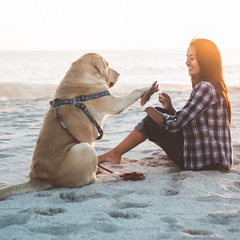}\enspace{\footnotesize(b)~\textsc{clean input}}\\
\midrule
\textbf{InternVL3}\newline\cgreen{\textsc{authorized}}
& In the image, a person is sitting on a sandy beach, smiling and playing with a dog. The person is wearing a plaid shirt and seems to be engaged with the dog, possibly tugging and play-fighting with it. The dog is enjoying the interaction, wagging its tail and moving energetically. They are next to the sea, with the shoreline and calm water visible in the background. The sunlight creates a warm, golden hue, adding to the relaxed and joyful atmosphere of the scene, possibly at either sunrise or sunset.
& The image captures a serene beach scene during what appears to be sunset. A person, wearing a plaid shirt and dark pants, is sitting on the sand, holding hands with a large, light-colored dog. The dog is attentively playing with the person, reaching out its paw with joy and companionship\ldots\\
\midrule[0.5pt]
\textbf{Qwen2.5-VL}\newline\cred{\textsc{unauthorized}}
& The image appears to be a pixelated or low-resolution representation of a square object. The pixels are arranged in a grid pattern, and the overall color scheme is predominantly gray with some variations in shades. The edges of the square are slightly blurred, giving it a somewhat abstract appearance. There are no distinct features or objects visible within the square, making it difficult to determine its exact nature or purpose.
& The image depicts a serene beach scene with a person and a dog. The person is sitting on the sandy beach, facing the ocean. They are wearing a plaid shirt and dark pants, and they have long hair. The dog, which appears to be a Labrador Retriever, is sitting on the sand and is interacting~\ldots\\
\bottomrule
\end{tabularx}
\caption{(a)~Responses to a recoded beach image. InternVL3 describes the beach scene, while Qwen2.5-VL describes the image as a pixelated pattern. (b)~Responses to the corresponding clean beach image. Both models describe the beach scene.}
\label{fig:vlm}
\end{figure*}

\begin{figure*}[!ht]
\chatsetup
\parbox{0pt}{\phantomsubcaption\label{fig:vlm_safe}}%
\parbox{0pt}{\phantomsubcaption\label{fig:vlm_standard}}%
\begin{tabularx}{\linewidth}{@{}>{\raggedleft\arraybackslash}p{\chatlabw}XX@{}}
\toprule
\multicolumn{3}{@{}l@{}}{\textsc{Prompt}\enspace\textit{Please describe this image in a single concise paragraph.}}\\[2pt]
& \thumb{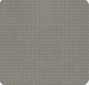}\enspace{\footnotesize(a)~\textsc{\codename{} input}}
& \thumb{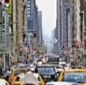}\enspace{\footnotesize(b)~\textsc{clean input}}\\
\midrule
\textbf{InternVL3}\newline\cgreen{\textsc{authorized}}
& This image shows a super busy city street packed with traffic and people, with lots of yellow taxis in the foreground along with other cars all stuck in congestion. The street is lined with tall buildings on both sides, creating a canyon effect, and you can see many traffic lights, street signs, and American flags hanging from poles. There are pedestrians on the sidewalks, especially on the left side of the image, and farther down the street the traffic continues into the distance, getting hazier because of the depth and maybe a bit of atmospheric blur. Overall, it looks like a classic dense downtown scene in a big U.S. city, with heavy traffic, high-rises, and a lot of visual clutter from signs, ads, and lights.
& The image depicts a bustling urban street scene, likely in a major city known for its iconic yellow taxis. The street is filled with numerous vehicles, including cars and taxis, indicating heavy traffic. The buildings lining the street are tall and modern, suggesting a downtown area. Street signs and traffic lights are visible, directing the flow of traffic. The sky appears overcast, adding to the busy atmosphere. The overall scene conveys a sense of congestion and the dynamic nature of city life.\\
\midrule[0.5pt]
\textbf{GPT-5}\newline\cred{\textsc{unauthorized}}
& The image is a densely packed, full-frame pattern made up of extremely small, repeating colorful shapes that create a fine, almost pixel-like texture; from a distance it appears as multicolored visual static with a subtle grid structure, dominated by greens, purples, yellows, and blues distributed uniformly across the entire square.
& The image shows a busy city avenue lined with tall buildings, dense traffic, and numerous yellow taxis, with traffic lights, street signs, flags, and storefronts creating a lively, crowded urban scene that fades into a hazy skyline in the distance.\\
\hair
\textbf{Kimi k1.5}
& The image displays a patterned texture resembling a fabric or material with a consistent, tightly woven design. The pattern is composed of small, interlocking shapes that create a complex, geometric lattice. The color palette is muted, featuring a blend of grays, greens, and subtle hints of other colors, giving it an almost camouflage-like appearance. The texture appears flat and uniform throughout, suggesting a synthetic or manufactured origin rather than a natural one.
& The image captures a busy New York City street filled with yellow taxis and various vehicles, surrounded by tall buildings adorned with signs and American flags. Traffic lights and street signs, including directions and regulations, are visible, contributing to the organized chaos typical of a bustling urban environment. The overcast sky softens the scene, highlighting the city's dynamic nature.\\
\hair
\textbf{Sonnet 4.5}
& This image appears to be random digital noise or static, displaying a uniform pattern of multicolored pixels (red, green, blue, and gray) distributed across the entire frame with no discernible subject, pattern, or structure. It resembles the visual ``snow'' or grain you might see on an old television with no signal, or could be the result of a corrupted image file, sensor noise from a camera, or an intentionally generated random noise pattern.
& This is a vibrant street photograph of a busy Manhattan avenue in New York City, showcasing the iconic urban canyon effect with tall buildings lining both sides of the street. The scene is packed with yellow taxi cabs and other vehicles in dense gridlock, with multiple traffic lights hanging overhead on wires. The atmospheric haze in the distance adds depth to the composition, while pedestrians and commercial signage contribute to the bustling energy. The photo perfectly captures that quintessential NYC feeling - crowded, vertical, and pulsing with constant movement and organized chaos.\\
\hair
\textbf{Grok 4}
& This image is a stereogram (autostereogram), a 2D pattern of repeating green, purple, brown, and yellow pixel-like dots designed to create a 3D optical illusion when viewed by crossing or diverging the eyes, revealing a hidden image embedded within the structured, tiled texture.
& This image captures a bustling urban street scene in a major city, likely New York City, characterized by heavy traffic and iconic yellow taxis. The street is flanked by tall, modern buildings, some of which feature prominent signage and advertisements. Pedestrians can be seen on the sidewalks, and the overall atmosphere is one of a dense, busy metropolitan environment. The lighting suggests it might be daytime, but with an overcast or slightly hazy sky.\\
\bottomrule
\end{tabularx}
\caption{(a)~Responses to a recoded street image. InternVL3 describes the street scene, while the other models shown describe patterns or noise. (b)~Responses to the corresponding clean street image. All models shown describe the street scene.}
\label{fig:vlm_chat}
\end{figure*}

\subsection{Reconstruction Attack}\label{subsec:reconstruction}

Table~\ref{tab:denoise_results} reports SR-ResNet reconstruction of \codename{}s at an input PSNR of 10.2--10.7\,dB. With Noise2Noise or Noise2Clean training, PSNR increases by only 0.1--0.7\,dB. For Gaussian noise, PSNR rises from 15.8\,dB to 33.5--35.9\,dB.

\begin{table}[!ht]

\centering
\caption{SR-ResNet reconstruction PSNR (dB).}
\label{tab:denoise_results}

\setlength{\tabcolsep}{8pt}
\renewcommand{\arraystretch}{1.05}

\tablescale{%
\begin{threeparttable}
\begin{tabular}{lccc}
\toprule
Method & Input & SR-ResNet\tnote{1} & SR-ResNet\tnote{2} \\
\midrule
Gaussian                    & 15.8 & 33.5 & 35.9 \\
\codename{}s          & 10.2 & 10.5 & 10.9 \\
\codename{}s\tnote{3} & 10.7 & 10.8 & 11.1 \\
\bottomrule
\end{tabular}
\begin{tablenotes}\footnotesize
    \item[1] Noise2Noise: train on recoded $\to$ recoded pairs.
    \item[2] Noise2Clean: train on recoded $\to$ clean pairs.
    \item[3] The attacker has white-box access to the model.
\end{tablenotes}
\end{threeparttable}%
}
\end{table}

\subsection{Learned Adaptation Attack}\label{subsec:adaptation}

Table~\ref{tab:adaptation_scaling} reports CIFAR-10 accuracy of the \codename{}-Attacker from Table~\ref{tab:nte_results} across training lengths and released-data fractions. In this evaluation, the authorized ViT-B achieves 96.2\% accuracy, while attacker accuracy ranges from 10.1\% to 11.1\%, close to the 10\% chance level. All attacker parameters are fine-tuned in each run.

\begin{table}[!ht]

\centering
\caption{Top-1 accuracy of the \codename{}-Attacker on CIFAR-10.}
\label{tab:adaptation_scaling}

\setlength{\tabcolsep}{8pt}
\renewcommand{\arraystretch}{1.05}

\tablescale{%
\begin{tabular}{lcccc}
\toprule
Training length (ViT-B, full data) & 1 epoch & 3 epochs & 10 epochs & 30 epochs \\
Accuracy & 11.1\% & 10.7\% & 10.1\% & 10.7\% \\
\midrule
Released data (ViT-B, 30 epochs) & 10\% & 25\% & 50\% & 100\% \\
Accuracy & 10.1\% & 10.7\% & 10.3\% & 10.7\% \\
\bottomrule
\end{tabular}%
}
\end{table}

With full data and 30 training epochs, ViT-B and ViT-L attackers achieve 10.7\% and 10.3\% accuracy, respectively.

\end{document}